\documentclass[sigconf]{acmart}
\usepackage{svg}
\usepackage{float}                  
\usepackage{subfig}                 
\usepackage{algorithm}
\usepackage{algpseudocode}
\usepackage{cleveref}
\usepackage{graphicx}

\newtheorem{theorem}{Theorem}

\newtheorem{corollary}{Corollary}

\newtheorem{assumption}{Assumption}

\usepackage{multirow}
\usepackage{multicol}
\usepackage{enumitem}
\definecolor{pink}{rgb}{0.858, 0.188, 0.478}
\definecolor{commentcolor}{RGB}{110,154,155}   
\AtBeginDocument{%
  \providecommand\BibTeX{{%
    \normalfont B\kern-0.5em{\scshape i\kern-0.25em b}\kern-0.8em\TeX}}}

\setcopyright{acmcopyright}
\copyrightyear{2024}
\acmYear{2024}
\acmDOI{XXXXXXX.XXXXXXX}

\acmConference[WSDM '24]{the 17th ACM International Conference on Web Search and Data Mining}{March 04--08, 2024}{Mérida, Mexico}
\acmPrice{15.00}
\acmISBN{978-1-4503-XXXX-X/18/06}

\begin{document}

\title{Rethinking and Simplifying Bootstrapped Graph Latents}

\author{Wangbin Sun}
\affiliation{%
  \institution{Sun Yat-sen University}
 \country{}
  }
\email{sunwb7@mail2.sysu.edu.cn}

\author{Jintang Li}
\affiliation{%
  \institution{Sun Yat-sen University}
 \country{}
  }
\email{lijt55@mail2.sysu.edu.cn}

\author{Liang Chen}
\authornote{Corresponding author.}
\affiliation{%
  \institution{Sun Yat-sen University}
 \country{}
}
\email{chenliang6@mail.sysu.edu.cn}

\author{Bingzhe Wu}
\affiliation{%
  \institution{Peking University}
 \country{}
}
\email{wubingzhe@pku.edu.cn}

\author{Yatao Bian}
\affiliation{%
  \institution{Tencent AI Lab}
 \country{}
}
\email{ yatao.bian@gmail.com}

\author{Zibin Zheng}
\affiliation{%
 \institution{Sun Yat-sen University}
 \country{}
 }
 \email{zhzibin@mail.sysu.edu.cn}


\begin{abstract}
Graph contrastive learning (GCL) has emerged as a representative paradigm in graph self-supervised learning, where negative samples are commonly regarded as the key to preventing model collapse and producing distinguishable representations. 
Recent studies have shown that GCL without negative samples can achieve state-of-the-art performance as well as scalability improvement, with bootstrapped graph latent (BGRL) as a prominent step forward. 
However, BGRL relies on a complex architecture to maintain the ability to scatter representations, and the underlying mechanisms enabling the success remain largely unexplored. 
In this paper, we introduce an instance-level decorrelation perspective to tackle the aforementioned issue and leverage it as a springboard to reveal the potential unnecessary model complexity within BGRL. 
Based on our findings, we present SGCL, a simple yet effective GCL framework that utilizes the outputs from two consecutive iterations as positive pairs, eliminating the negative samples. 
SGCL only requires a single graph augmentation and a single graph encoder without additional parameters. 
Extensive experiments conducted on various graph benchmarks demonstrate that SGCL can achieve competitive performance with fewer parameters, lower time and space costs, and significant convergence speedup.\footnote{Code is made publicly available at \url{https://github.com/ZsZsZs25/SGCL}.}
\end{abstract}

\ccsdesc[500]{Information systems~Data mining}
\ccsdesc[500]{Computing methodologies~Unsupervised learning}

\keywords{Graph Neural Networks; 
    Graph Representation Learning; 
    Graph Self-supervised Learning
}


\maketitle

\section{Introduction}

\begin{figure}[h]
    \centering
    \includegraphics[width=1\linewidth]{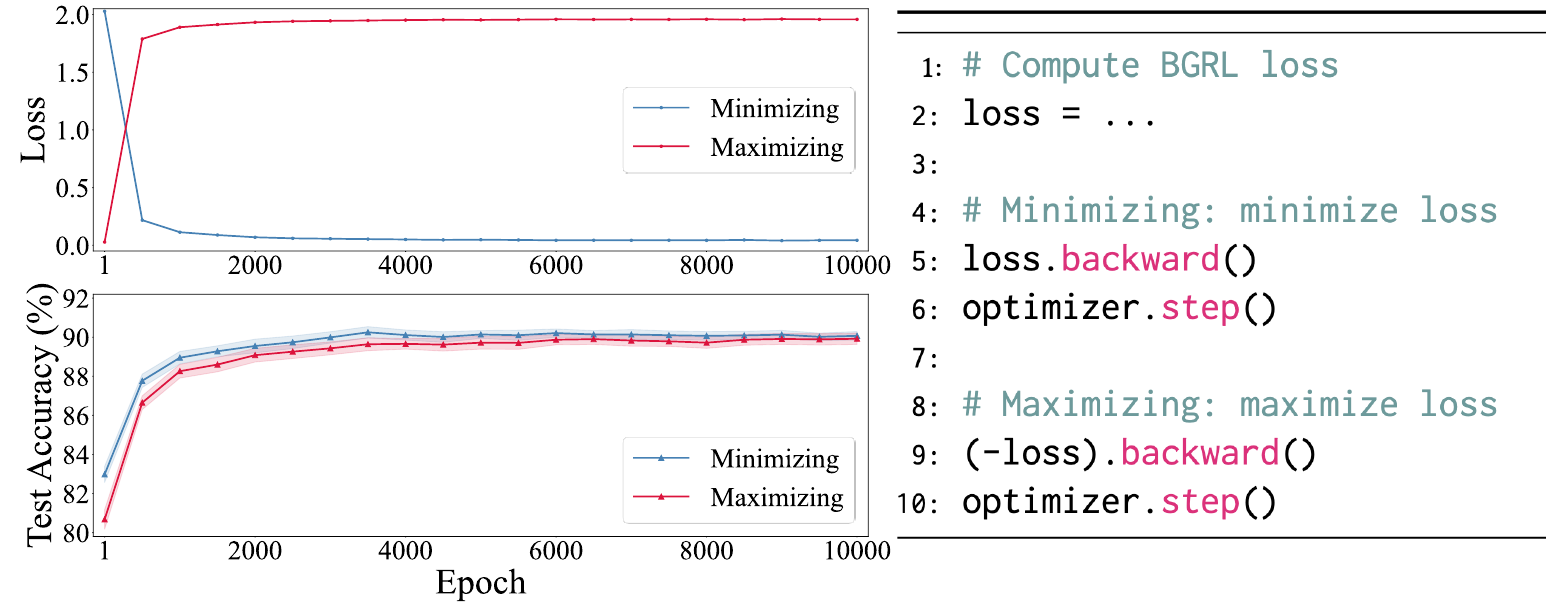}
    \caption{Training loss and test accuracy curves of maximizing and minimizing loss on the  Amazon-Computers dataset.}
    \label{img:neg_loss_comp}
    
\end{figure}

Graph self-supervised learning (GSSL), which learns meaningful representations through purpose-designed pretext tasks, has gained prominence as a potent approach to mitigate the pervasive issue of artificial label dependency~\cite{GSSL_P_C_G}. Drawing inspiration from contrastive learning in vision research~\cite{MoCo, SimCLR, BYOL, Siamese}, graph contrastive learning (GCL) has emerged as a prevailing paradigm of GSSL and exhibited remarkable success across various downstream tasks, including citation classification~\cite{DGI, MVGRL, GRACE, BGRL,CCA-SSG}, recommendation systems~\cite{SGL, SSL_Hyperrec}, and molecular property prediction~\cite{MoICLR, Motif}.

Typically, GCL endeavors to congregate positive pairs to be invariant to noise (alignment) and achieve a roughly uniform distribution of representations through negative pairs (uniformity), where the uniformity property serves as a critical factor in preventing model collapse and generating discriminative representations~\cite{algin_uniform}. As such, current GCL methods inherently rely on increasing the number and quality of negative samples, leading to heavy computation and memory overhead, especially for large graphs~\cite{BGRL}. To overcome these issues, researchers have explored the possibility of learning without negative samples, and recent advances have demonstrated its existence~\cite{BGRL, CCA-SSG, AFGRL}. As pioneers on this path, bootstrapped graph latents (BGRL) and its variants~\cite{BGRL,AFGRL} have evolved into state-of-the-art on various downstream benchmarks and have also shown good scalability.

Basically, BGRL employs additional network modules to scatter output representations and alleviate model collapses, such as distinct graph augmentation functions, predictor networks, and asymmetric networks. It then learns node representations by predicting alternative augmentations of the input graph and maximizing the similarity between the prediction and paired target.  
However, the underlying nature of the success of such a complex architecture is not yet fully explored. 
For instance, in the absence of negative samples, we are supposed to maximize the similarity between positive pairs, corresponding to minimizing the loss function during the training process. 
Surprisingly, we find that BGRL still works well even when minimizing the similarity of positive samples, which is equivalent to maximizing the loss function. As shown in Figure~\ref{img:neg_loss_comp}, the trend of test accuracy is almost the same for minimizing and maximizing the loss function during the training process on the Amazon-Computers dataset, and results on other datasets that are omitted for space show exactly a similar trend. 
This phenomenon greatly challenges our traditional understanding, and a neglected issue raises in our minds: \emph{what in the complex architecture truly contributes to the success of BGRL?}

To answer this question, we empirically and theoretically analyze the nature of BGRL's success. Empirically, we investigate the role of components in BGRL and find that the existence of graph augmentations and predictor is fundamental to BGRL, regardless of whether distinct augmentation functions and asymmetric networks are applied or whether the similarity of positive pairs is maximized. 
Theoretically, we reveal that the predictor implicitly assists BGRL in an instance-level decorrelation way, which is the cornerstone for BGRL to generate discriminative representations and prevent model collapses. 
Nevertheless, achieving the goal of decorrelation through optimizing parameterized predictor may result in slower model convergence and subsequently impact model performance, especially in large-scale graphs. 
To address this issue, we estimate the predictor from the output of the graph encoder without additional learning parameters.
The above findings suggest the substantial redundancy in BGRL. Therefore, in this paper, we are motivated to design a simple yet effective GCL framework named SGCL, which only requires a single graph augmentation function, a single graph encoder and a non-parametric predictor. In particular, we adopt a pipeline-style training paradigm, where we only perform one augmentation operation each iteration and take the outputs of two consecutive iterations as positive samples.
As shown in Table~\ref{tab:tech comparison}, the proposed lightweight SGCL does not rely on negative pairs, an additional discriminator, projector, or predictor while only requiring one augmentation operation at each iteration. To summarize, this work makes the following main contributions: 
\begin{itemize}
    \item
          We present a counterintuitive observation of the classical negative-sample-free GCL framework BGRL, i.e., making positive pairs dissimilar still works well, which could motivate future research to explore why GCL works.
    \item
          We provide both experimental and theoretical analysis of BGRL, uncovering the hidden factors for its success and the redundancy in its architecture.
    \item
          We propose SGCL, a simple and effective negative-sample-free GCL method, that maximizes the similarity of positive pairs from consecutive iterations using only one graph augmentation, one graph encoder, and one inferential predictor.
    \item
          Extensive experiments demonstrate that SGCL could achieve competitive performance compared to BGRL and state-of-the-arts with fewer parameters, less memory, and faster running and convergence speed.          
\end{itemize}

\begin{table}[t]
    \caption{Technical comparison with previous methods. \emph{Neg samples}: require negative samples or not. \emph{Proj/Pred/Disc}: require projector/predictor/discriminator or not. \emph{\#Encoder}: number of graph encoders. \emph{\#Aug View}: number of augmented graph views at each iteration.}
    \label{tab:tech comparison}
    \scalebox{0.8}{
        \begin{tabular}{lcccc}
            \toprule
            \textbf{Methods}   & Neg samples & Proj/Pred/Disc & \#Encoder & \#Aug View \\
            \midrule
            \textbf{DGI}\cite{DGI}       & \checkmark  & \checkmark     & 1         & 1          \\
            \textbf{MVGRL}\cite{MVGRL}     & \checkmark  & \checkmark     & 2         & 2          \\
            \textbf{GRACE}\cite{GRACE}     & \checkmark  & \checkmark     & 1         & 2          \\
            \textbf{GCA}\cite{GCA}       & \checkmark  & \checkmark     & 1         & 2          \\
            \textbf{COSTA}\cite{COSTA}     & \checkmark  & \checkmark     & 1-2       & 1-2        \\
            \textbf{BGRL}\cite{BGRL}      & -           & \checkmark     & 2         & 2          \\
            \textbf{AFGRL}\cite{AFGRL}     & -           & \checkmark     & 2         & 0          \\
            \textbf{CCA-SSG}\cite{CCA-SSG}   & -           & -              & 1         & 2          \\
            \textbf{SGCL} (ours) & -           & -              & 1         & 1          \\
            \bottomrule
        \end{tabular}}
\end{table}
\section{Related Work}

Our work is conceptually related to graph neural networks, graph contrastive learning, and recent advancements in contrastive learning without using negative examples. We proceed by reviewing major threads of relevant research efforts.\\
\textbf{Graph neural networks.}
Graph neural networks (GNNs) are a class of neural networks that are widely adopted as encoders for representing graph data. They generally follow the canonical \emph{message passing} scheme that each node’s representation is computed recursively by aggregating representations (``messages'') from its immediate neighbors~\cite{GCN, GraphSAGE}. So far, extensive studies have been conducted on GNNs for a variety of graph analysis tasks and achieved significant improvements over traditional methods on benchmarks. Recently, there has been significant interest in simplifying the GNN architectures by dropping non-linear activation~\cite{SGC,LightGCN} and knowledge distillation~\cite{DBLP:conf/iclr/ZhangLSS22}, which paves a clearer path towards improving the scalability of GNNs on large-scale graphs.\\
\textbf{Graph contrastive learning.}
Contrastive learning on graphs is an important technique to make use of rich unlabeled data. Such a paradigm typically learns representations from self-defined supervisions by contrasting positive
and negative samples from different augmentation views of inputs. As a pioneer work, DGI~\cite{DGI} proposes to learn node representations through contrasting local and global embeddings. 
GRACE \cite{GRACE} and GCA \cite{GCA} learn node representations by pulling the representation of the same node in two augmented views closer while pushing the representations of the other nodes in two views further. 
Despite the success of contrastive learning on graphs, they require a large number of negative samples with carefully crafted encoders and augmentation techniques to learn discriminative representations, making them suffer seriously from high computation and memory overheads during training~\cite{AFGRL,CCA-SSG}.\\
\textbf{Graph contrastive learning without negative samples.}
Recently, literature has shown that negative samples are not always necessary for graph contrastive learning~\cite{BGRL,CCA-SSG,AFGRL}. Typically, contrastive learning with no negative samples relies on mechanisms like stop-gradient~\cite{BGRL,AFGRL}, additional predictor \cite{BGRL,AFGRL}, and feature decorrelation loss function~\cite{CCA-SSG} to avoid representation collapse. Among contemporary approaches, BGRL is a promising recent alternative to negative-sample-free GCL algorithms, leading to new state-of-the-art performance on broad downstream tasks. Yet, it requires complex asymmetric architectures to refine node representations, which can seriously compromise the scalability of the method. In this work, we dig into the hidden reasoning behind the key success of BGRL and seek to simplify its model design with a theoretical analysis of the model effectiveness.

\section{Problem Statement and Preliminary}

\subsection{Problem Statement}
Consider a graph $\mathcal{G}=(\mathcal{V}, \mathcal{E})$, where $\mathcal{V}=\{v_1, v_2, \dots, v_N\}$ and $\mathcal{E} \subseteq \mathcal{V} \times \mathcal{V}$ denote the node set and edge set respectively, $N=|\mathcal{V}|$ is the number of nodes.  $\mathbf{A} \in \mathbb{R}^{N \times N}$ and $\mathbf{X} \in \mathbb{R}^{N \times F}$ are the associated input adjacency matrix and feature matrix with $\mathcal{G}$. We are committed to learning a graph encoder $f_{\theta}(\cdot)$ to obtain the low-dimensional node embeddings $\mathbf{H}=f_{\theta}(\mathbf{A}, \mathbf{X}) \in \mathbb{R}^{N \times d}$ without accessing any label information, where $d$ is the embedding size.

\subsection{Bootstrapped Graph Latents}
We first introduce the Bootstrapped Graph Latents (BGRL) ~\cite{BGRL}, which aims to maximize the similarity between the representations of the same node produced from two distinct augmented graph views in virtue of the following three major components. \\
\textbf{Graph augmentation.} 
Given the adjacency matrix $\mathbf{A}$ and feature matrix $\mathbf{X}$ of a graph $\mathcal{G}$, BGRL utilizes two stochastic graph augmentation functions $\mathcal{T}_1$ and $\mathcal {T}_2$ to produce two alternate graph views $\mathcal{G}_1 \sim (\widetilde{\mathbf{A}}_1, \widetilde{\mathbf{X}}_1)$ and $\mathcal{G}_2 \sim ( \widetilde{\mathbf{A}}_2, \widetilde{\mathbf{X}}_2)$ at each training iteration. Specifically, the augmentation functions $\mathcal{T}_1$ and $\mathcal{T}_2$ are simple combinations of random node feature masking and edge masking~\cite{GRACE} with favorable masking probabilities. \\
\textbf{Node embedding generation.} 
Varying from the classical GCL frameworks with a shared graph encoder, BGRL adopts two separate graph encoders, i.e., the online encoder $f_{\theta}$ and target encoder $f_{\phi}$. The two augmented graph views $\mathcal{G}_1$ and $\mathcal{G}_2$ are fed into the online encoder and target encoder respectively to produce online representations $\widetilde{\mathbf{H}}_1=f_{\theta}(\widetilde{\mathbf{A}}_1, \widetilde{\mathbf{X}}_1)$ and target representations $\widetilde{\mathbf{H}}_2=f_{\phi}(\widetilde{\mathbf{A}}_2, \widetilde{\mathbf{X}}_2)$.
Moreover, BGRL applies an additional node-level predictor (default as a MLP) $p_{\theta}$ to transform the online representations to a prediction $\widetilde{\mathbf{Z}}_{1}=p_{\theta}(\widetilde{{\mathbf{H}}}_1)$ of the target representations $\widetilde{\mathbf{H}}_2$. \\
\textbf{Similarity maximization.} 
Since BGRL is negative-sample-free, it learns by maximizing the cosine similarity between the prediction of target representations $\widetilde{\mathbf{Z}}_{(1,i)}$ and the true target representations $\widetilde{\mathbf{H}}_{(2,i)}$, i.e., positive pairs.
The objective function is defined as
\begin{equation}\label{eq:bgrl_loss}
    \ell(\theta, \phi)=2-\frac{2}{N} \sum_{i=0}^{N-1} \frac{\widetilde{\mathbf{Z}}_{(1, i)} \widetilde{\mathbf{H}}_{(2, i)}^{\top}}{||\widetilde{\mathbf{Z}}_{(1, i)}||_2 ||\widetilde{\mathbf{H}}_{(2, i)}||_2},
\end{equation}
where $\widetilde{\mathbf{Z}}_{(1, i)}\in\mathbb{R}^{d}$, $\widetilde{\mathbf{H}}_{(2, i)}\in\mathbb{R}^{d}$ and $||\cdot||_2$ is the $\ell_{2}$ vector norm operation.
It's worth noting that only the parameters of online encoder $f_{\theta}$ and predictor $p_{\theta}$ are updated with respected to the gradients from the objective function while the target encoder parameters $f_{\phi}$ are updated as an exponential moving average (EMA) of $f_{\theta}$ with a decay rate $\tau$, i.e., $f_{\phi} = \tau f_{\phi} + (1 - \tau) f_{\theta}$. Therefore, BGRL takes the outputs from the ensemble optimized parameters as targets to enhance the model step by step, which is a technique also referred to as bootstrapping.
 
\begin{figure}[t]
    \centering
    \includegraphics[width=1\linewidth]{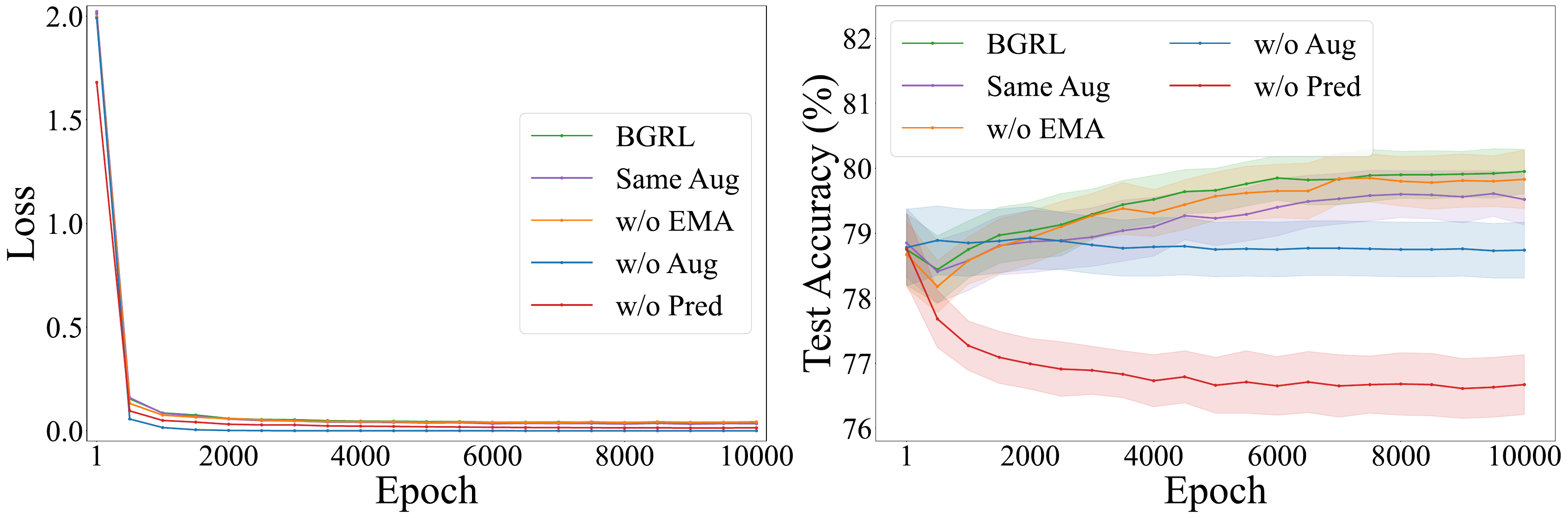}
    \caption{Loss and accuracy curves of BGRL and four variants on WikiCS.  \emph{Same Aug}: unifying graph augmentations, i.e., $\mathcal{T}_{1}=\mathcal{T}_{2}$. 
    \emph{w/o EMA}: removing EMA, i.e., $\tau=0$.
    \emph{w/o Aug}: removing graph augmentations.
    \emph{w/o Pred}: removing predictor.}
    \label{img:module_cmp_wikics}
\end{figure}

\section{Motivation}\label{sec:motivation}
In this section, we conduct an empirical exploration of BGRL to clarify the contributions of different modules in the framework and give theoretical insights into the nature of its success, which is attributed to its implicit instance-level decorrelation operation between the output representations of GNNs. Moreover, based on the aforementioned analysis, we also reveal the redundancy in BGRL. This discovery inspires us to further simplify the framework.

\subsection{Empirical Exploration}\label{sec:emp_motivation}

On the whole, BGRL contains the following components: graph augmentations, online encoder, target encoder, EMA, predictor, and cosine similarity loss function. For further understanding and simplification, at the early beginning, we perform corresponding ablation experiments to obtain an intuitive understanding of the role of the above components first.
From Figure~\ref{img:module_cmp_wikics}, we can draw the following two conclusions: (1) The key to the success of BGRL lies in graph augmentation and predictor. When only removing graph augmentations, the model performance drops rapidly and maintains almost a straight line throughout the entire training process. Even more, BGRL fails to learn information when only the predictor is removed, as the test accuracy keeps decreasing. (2) The BGRL framework exhibits redundancy. Overall, the contribution of using distinct augmentation functions or EMA mechanisms to BGRL is negligible, as both the loss and test accuracy curves exhibit a highly similar trend to the vanilla BGRL.
For detailed discussions about EMA, we refer readers of interest to Appendix \ref{appdix:ema_explain}.

The interpretation of graph augmentations is intuitive. Appropriate graph augmentations could help the graph encoder explore richer underlying semantic information of graphs~\cite{gssl_survey}. Accordingly, removing graph augmentations results in less learnable information and consequently a less inspired model, which is consistent with the more rapidly decreasing trend and lower bound of the training loss as well as the struggling test accuracy improvement in Figure~\ref{img:module_cmp_wikics}. 
However, the behavior of the predictor and loss function is still puzzling, which leads to the following theoretical analysis.

\subsection{Theoretical Analysis}\label{sec:thero_motivation}

\textbf{Assumptions.} 
Before theoretical analysis, we first introduce the linearity assumption regarding predictor $p_{\theta}$. Moreover, based on experimental observations, we assume a certain relationship between the online representations $\widetilde{\mathbf{H}}_1$ and target representations $\widetilde{\mathbf{H}}_2$.

\begin{assumption}\label{assump:lnear_p}
    (Linearity of predictor $p_{\theta})$:
    \begin{equation}\label{eq:linear_p}
        p_{\theta}(\widetilde{{\mathbf{H}}}_{(1,i)})=\mathbf{W}_{p}\widetilde{{\mathbf{H}}}_{(1,i)},
        where \text{\space} \mathbf{W}_p \in \mathbb{R}^{d \times d}.
    \end{equation}
\end{assumption}

\begin{assumption}\label{assump:same_repre}
    When optimizing by Eq.\eqref{eq:bgrl_loss}, the online representation and target representation of node $i$ progressively meet
    \begin{equation}\label{eq:same_repre}
        \widetilde{\mathbf{H}}_{(1, i)}=m_{i}\widetilde{\mathbf{H}}_{(2, i)},
        where \text{\space} m_{i}>0.
    \end{equation}
\end{assumption}

Under Assumption~\ref{assump:lnear_p}, we regard the parameters of the predictor network as a simple linear network, which is a commonly used simplification technique in the analysis process~\cite{SGC, GRACE, GCA}. 
Assumption~\ref{assump:same_repre} is motivated by the experimental results presented in Figure~\ref{img:same_rep_wikics_comp}, where we report the average cosine similarity $\bar{s}$ and average euclidean distance $\bar{d}$ between $\widetilde{\mathbf{H}}_{(1, i)}$ and $\widetilde{\mathbf{H}}_{(2, i)}$ for all nodes. Formally, we have the following formulas 
\begin{equation}
    \bar{s}=\frac{1}{N}\sum_{i=1}^{N}\frac{\widetilde{\mathbf{H}}_{(1, i)} \widetilde{\mathbf{H}}_{(2, i)}^{\top}}{||\widetilde{\mathbf{H}}_{(1, i)}||_2 ||\widetilde{\mathbf{H}}_{(2, i)}||_2},
\end{equation}
\begin{equation}
    \bar{d}=\frac{1}{N}\sum_{i=1}^{N}|| \widetilde{\mathbf{H}}_{(1, i)} - \widetilde{\mathbf{H}}_{(2, i)} ||_{2}.
\end{equation}

As we can see from Figure~\ref{img:same_rep_wikics_comp}, during the training process, $\bar{s}$ progressively converges to 1 and $\bar{d}$ decreases to a stable and small value, indicating that $\widetilde{\mathbf{H}}_{(1, i)}$ and $\widetilde{\mathbf{H}}_{(2, i)}$ share the same geometric direction but differ in vector length. \\
\textbf{Instance-level decorrelation.} 
Combining the two assumptions, we can further unravel how the predictor enables BGRL to produce discriminative representations without negative samples, i.e., the instance-level decorrelation.

\begin{figure}[t]
    \centering
	\subfloat[WikiCS]{\includegraphics[width=0.47\columnwidth]{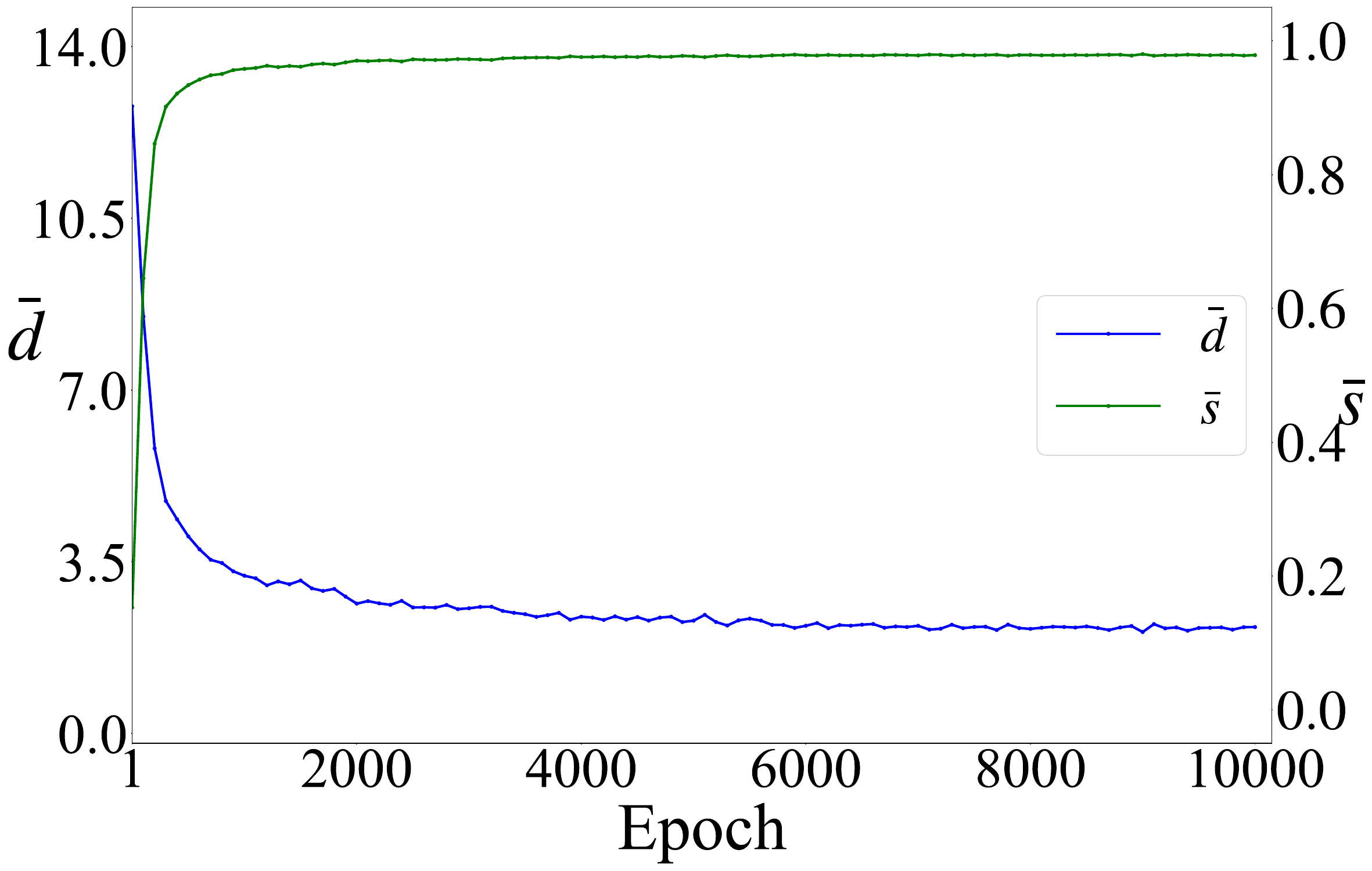}}\hspace{3.5pt}
	\subfloat[Amazon-Computers]{\includegraphics[width=0.47\columnwidth]{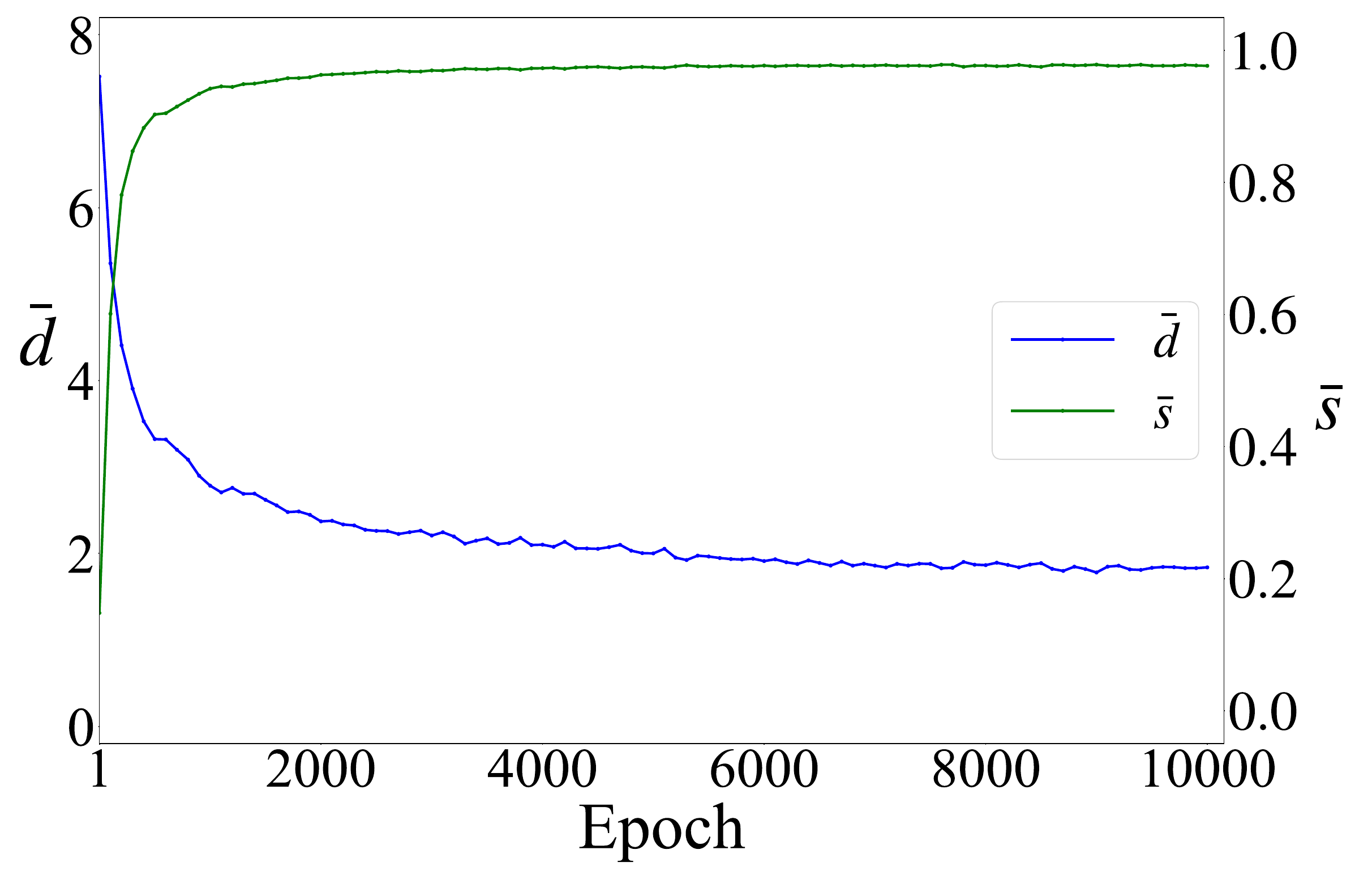}}
    \caption{Average cosine similarity $\bar{s}$ and average distance $\bar{d}$ between $\widetilde{\mathbf{H}}_{(1, i)}$ and $\widetilde{\mathbf{H}}_{(2, i)}$ for all nodes during training.}
    \label{img:same_rep_wikics_comp}
\end{figure}

\begin{corollary}\label{coro:z1_h1}
    Online representation $\widetilde{\mathbf{H}}_{(1, i)}$ for node $i$ is the eigenvector  of \space $\mathbf{W}_p$ with corresponding eigenvalue $\lambda_{i}$. Formally, we have
    \begin{equation}\label{eq:threom1}
        \mathbf{W}_p \widetilde{\mathbf{H}}_{(1, i)}=\lambda_{i} \widetilde{\mathbf{H}}_{(1, i)}, \text{\space} where \text{\space} \lambda_{i}>0.
    \end{equation}
\end{corollary}

The proof is presented in Appendix \ref{proof:coro_z1_h1}.
Note that $\mathbf{W}_p$ is a $d \times d$ real square matrix and the number of corresponding eigenvalues $k$ satisfies $0<k\le d$. Therefore, the predictor implicitly performs rough node classification and the classes are linearly independent eigenvectors associated with the distinct eigenvalues of $\mathbf{W}_p$, i.e., the instance-level decorrelation.
We argue that instance-level decorrelation is essential to decrease the correlation between output representations and produce distinguishable inputs for downstream tasks.
In addition, corollary~\ref{coro:z1_h1} still holds when minimizing the cosine similarity under the above assumptions with opposite eigenvalues, i.e.,
$\mathbf{W}_p \widetilde{\mathbf{H}}_{(1, i)}=-\lambda_{i} \widetilde{\mathbf{H}}_{(1, i)} \text{\space}$. That is, what the loss function essentially do is to align the geometric directions between prediction and target.
However, removing the predictor is equivalent to setting $\mathbf{W}_p=\mathbf{I}$, i.e., identity matrix with single eigenvalue 1, resulting in the node representations being required to evolve into eigenvectors of the same eigenvalue. In this way, the correlation of node representations increases, leading to the distinguishability reduction of representations and performance degeneration shown in Figure \ref{img:module_cmp_wikics}. 

\begin{figure}[t]
    \centering
	\subfloat[without predictor]{\includegraphics[width=.33\columnwidth]{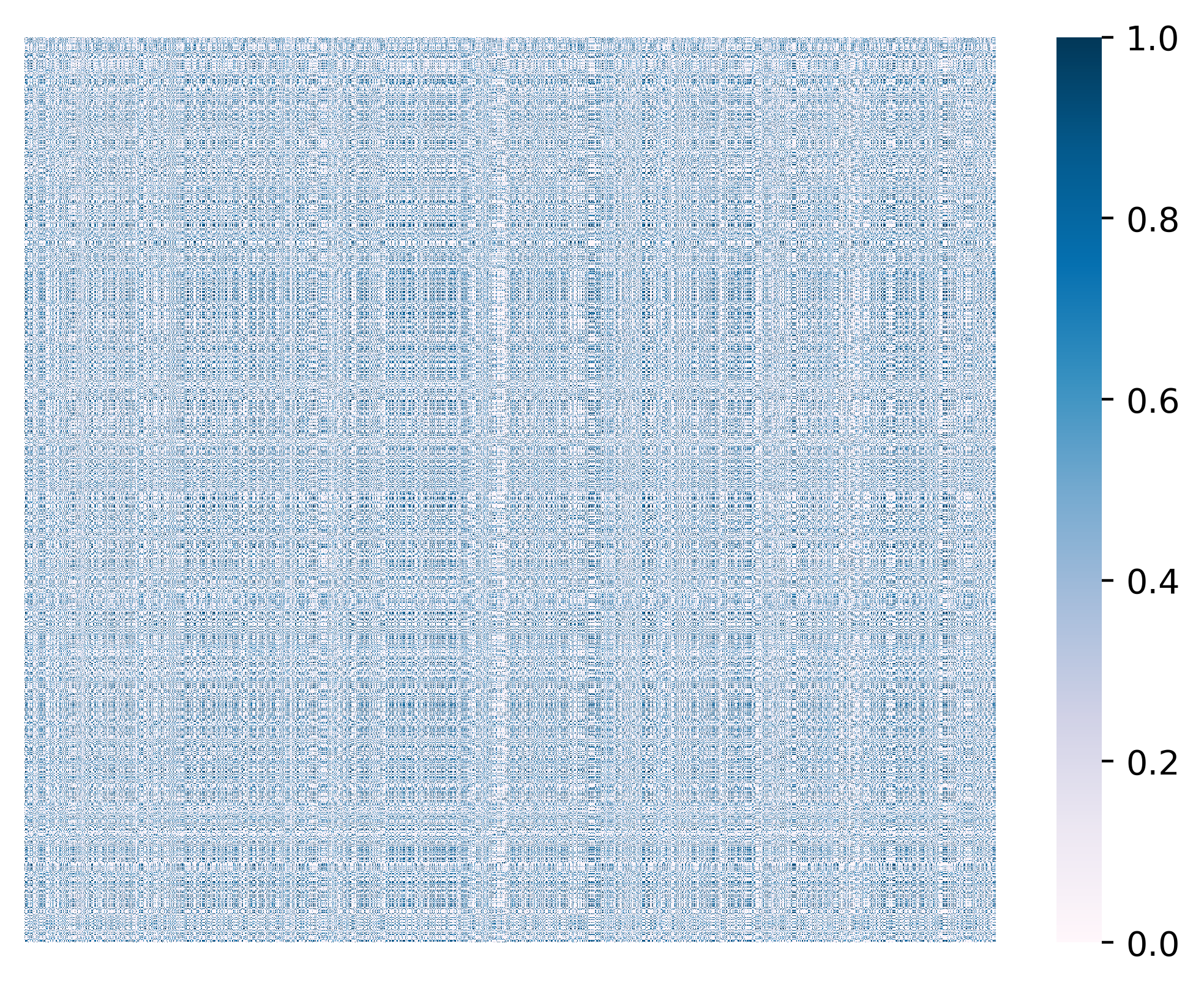}}
	\subfloat[minimizing similarity]{\includegraphics[width=.33\columnwidth]{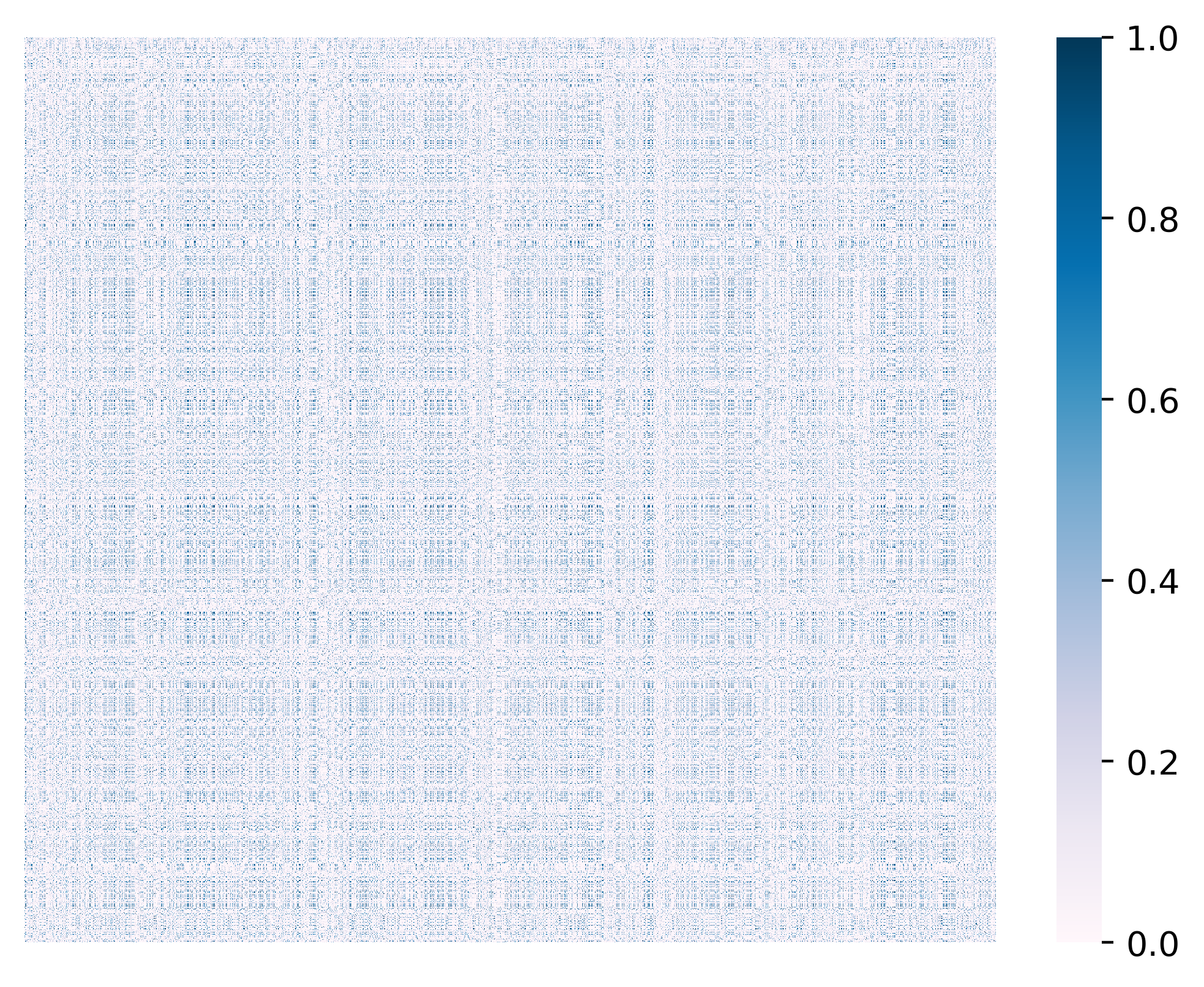}}
	\subfloat[default BGRL]{\includegraphics[width=.33\columnwidth]{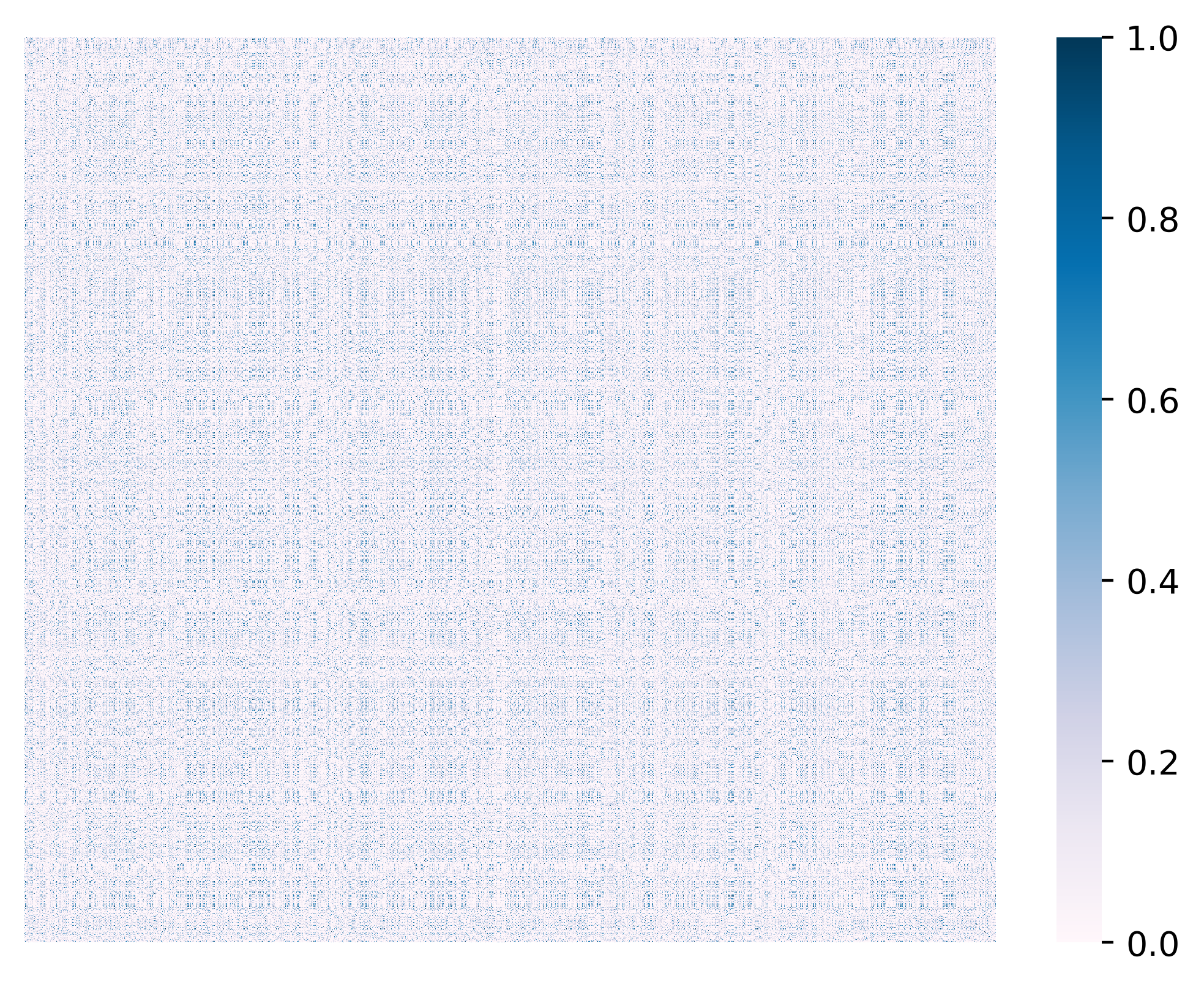}}
    \caption{Pearson correlation coefficient between different node representations when removing predictor (a), minimizing the similarity of positive pairs (b), and remaining the predictor and maximizing the similarity as the default setting of BGRL (c) on Amazon-Computers.}
    \label{img:node_corr_wikics_comp}
\end{figure} 

To further support our hypothesis, we provide the visualizations of the pearson correlation coefficient matrix of whether to use the predictor or minimize the similarity on the instance level in Figure~\ref{img:node_corr_wikics_comp}. As we can see, the correlation coefficient between different node representations is significantly lower than without the predictor, as long as the predictor is retained no matter whether the cosine similarity is maximized or minimized. \\
\textbf{Predictor inference.} 
The above analysis illustrates the predictor is a decisive factor to enable GNNs to learn distinguishable representations by instance-level decorrelation. However, we notice that learning predictor parameters to achieve decorrelation can result in slow convergence of GNNs, leading to suboptimal performance, especially in large-scale graphs. In Section~\ref{sec:conv_speed}, we present corresponding experimental results. Fortunately, Corollary~\ref{coro:z1_h1} indicates a connection between the predictor and the outputs of GNNs. We further show that the predictor can be directly from the covariance matrix of node representations without parameters, thus accelerating convergence speed and reducing parameters of BGRL.
\begin{theorem}\label{theo:infer_predictor}
    Suppose $\widetilde{\mathbf{H}}_1=\widetilde{\mathbf{H}}_2=\mathbf{H}$ following the zero-mean distribution and the covariance matrix of $\mathbf{H}$ satisfies the singular value decomposition (SVD) $\mathbf{\sum}=\frac{1}{N-1}\mathbf{H}^{\top}\mathbf{H}=\hat{\mathbf{U}}\hat{\mathbf{S}}\hat{\mathbf{V}}^{\top}$.
    When $\mathbf{W}_p$ is initialized as $\mathbf{W}_p=\epsilon \hat{\mathbf{U}} \hat{\mathbf{V}}^{\top}$, where all student singular values are $\epsilon$. As the training progresses, we have
    \begin{equation}
        \mathbf{W}_p = \frac{1}{N-1}\mathbf{H}^{\top}\mathbf{H}.
    \end{equation}
\end{theorem}

The proof  is presented in Appendix \ref{proof:theo_infer_predictor}.
Here $\mathbf{H}$ is $\ell_{2}$-normalized.
Theorem~\ref{theo:infer_predictor} provides a path to obtain the prediction representations without any parameters. For the case of non-linear predictor and unequal online and target representations, we leave it for future work.
Eventually, we come to the conclusion that the predictor, graph augmentations, and geometric direction alignment are crucial for BGRL, which leads us to the subsequent simplifications.  \\
\textbf{Further discussion on decorrelation.}
In graphs, the design of BGRL has been continuously referenced but lacks in-depth exploration~\cite{AFGRL,LaGraph,BGRL}.
Here, we demonstrate how BGRL scatter representations from the instance-level decorrelation perspective. Likewise, other works such as CCA-SSG~\cite{CCA-SSG} and its predecessor Barlow Twins~\cite{BarlowTinws} in images study dimension-level decorrelation that reduces the inter-dimension correlation, which may not perform well on low-dimensional datasets. Regardless, these methods share a common underlying mechanism, i.e., decorrelation.
Hence, exploring or combining the aforementioned methods remains highly promising and contributes significantly to the graph community.

\begin{figure*}[t]
    \centering
    \includegraphics[width=1\linewidth]{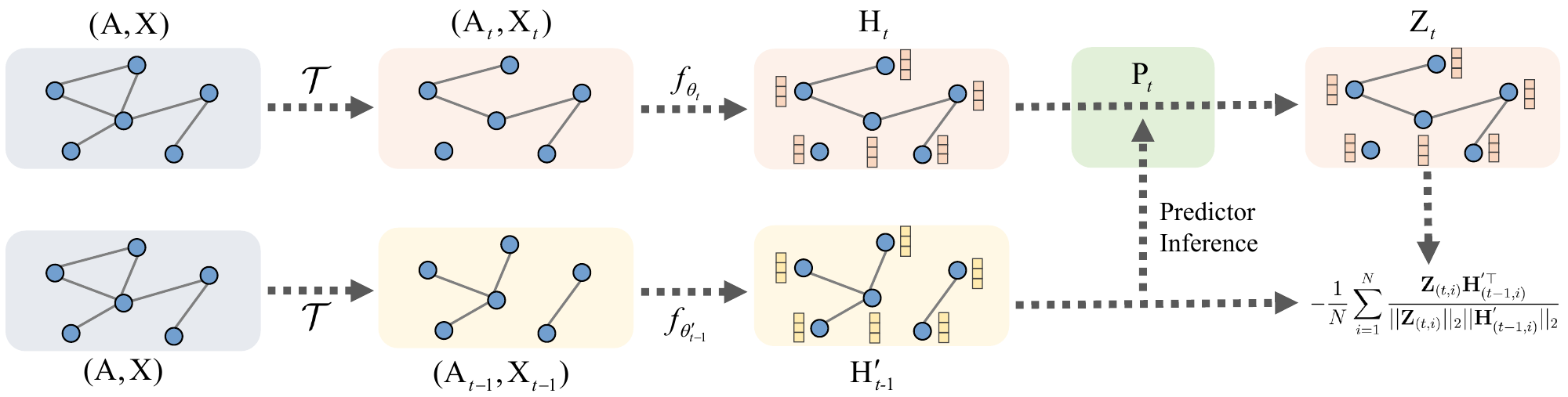}
    \caption{Illustration of the proposed SGCL framework. During each training iteration $t$, we only augment the original graph once and feed the augmented graph $(\mathbf{A}_{t},\mathbf{X}_{t})$ into GNN with parameters $\theta_{t}$ to obtain node representations $\mathbf{H}_{t}$. We use the predictor to generate prediction $\mathbf{Z}_{t}$ of the node representations $\mathbf{H}_{t-1}^{'}$ produced from the optimized parameters $\theta_{t-1}^{'}$ in previous iteration $t-1$, where the predictor is directly computed from $\mathbf{H}_{t-1}^{'}$. Finally, the similarity between $\mathbf{Z}_{t}$ and $\mathbf{H}_{t-1}^{'}$ is maximized. }
    \label{framework}
\end{figure*}

\section{Method}
Motivated by Section~\ref{sec:motivation}, in this section, we present the proposed simple framework SGCL in detail.
As illustrated in Figure \ref{framework}, SGCL is a compact \emph{pipeline} framework, that comprises only one graph augmentation function and one graph encoder without any other parameters. Contrary to BGRL, which optimizes graph encoders via maximizing the similarity between the prediction and target produced from two augmented views at each iteration, SGCL performs the agreement games between the prediction of node representations produced from iteration $t$ and pure node representations from iteration $t-1$. In the remainder of this section, we will first introduce the details of SGCL from the following four major components and end with a technical comparison with previous GCL methods.

\subsection{Graph Augmentation}
As a crucial role in boosting model performance, GCL methods, including BGRL, primarily maintain the same augmentation rule. That is, they generate two graph views from two distinct augmentation functions at each iteration. A slight difference is that BGRL only requires positive pairs, i.e., prediction and target representations. However, as demonstrated in Section~\ref{sec:emp_motivation}, applying two distinct augmentation functions is unnecessary. Moreover, we argue that generating two views at each training iteration may be redundant for BGRL since the target representations are replaceable by the previous outputs, as described in Section~\ref{sec:encoder}. Therefore, we reduce the graph augmentation operations in half and only produce one augmented view at each iteration in our framework.

For the augmentation strategies, we combine two straightforward and practical graph augmentation operations~\cite{BGRL}, i.e., feature masking and edge perturbation to set up our augmentation function $\mathcal{T}$.
In particular, at each training iteration $t$, given a graph $\mathcal{G}=(\mathcal{V}, \mathcal{E})$ with associated feature matrix $\mathbf{X}$ and adjacency matrix $\mathbf{A}$, we randomly drop a portion of edges and node feature dimensions following a specific Bernoulli distribution respectively,  
\begin{equation}\label{eq:drop_edge}
    \mathcal{E}_t = \text{Bernoulli}(\mathcal{E}, 1 - p_{e}), \space 0 < p_{e} < 1,
\end{equation}
\begin{equation}\label{eq:drop_feat}
    \mathbf{X}_t = \text{Bernoulli}(\mathbf{X}, 1 - p_{f}), \space 0 < p_{f} < 1,
\end{equation}
where $p_{e}$ and $p_{f}$ is the drop ratio for edges and feature dimensions. $(\mathbf{A}_{t}, \mathbf{X}_{t})$ is the corresponding input for the graph encoder.

\subsection{Graph Encoder} \label{sec:encoder}
Inspired by Section~\ref{sec:emp_motivation}, we remove the EMA of BGRL, resulting in $f_{\phi}=f_{\theta}$. Kindly note that only $f_{\theta}$ is updated by the gradient and $f_{\phi}$ can be regarded as the one-time backup of the optimized parameters of $f_{\theta}$ at each iteration.
Therefore, we propose to reduce the online encoder and target encoder to one single encoder.
To construct the online and target representations, we regard the encoder $f_{\theta_{t}}$ of current iteration $t$ as the online encoder and the optimized encoder $f_{\theta_{t-1}^{'}}$ of the previous iteration $t-1$ as the target encoder. Note that the superscript in $\theta_{t-1}^{'}$ means the gradient is stopped through this computational branch. Then, we take the augmented views $(\mathbf{A}_t, \mathbf{X}_t)$ and $(\mathbf{A}_{t-1}, \mathbf{X}_{t-1})$ as the input of corresponding graph encoders.
In specific, at iteration $t$, we feed the single augmented view $(\mathbf{A}_{t}, \mathbf{X}_{t})$ into the graph encoder $f_{\theta_{t}}$ to obtain node representations $\mathbf{H}_{t}=f_{\theta_{t}}(\mathbf{A}_{t}, \mathbf{X}_{t})$, which are referred to as \emph{online} representations. For \emph{target} representations, we adopt the node representations produced from iteration $t-1$ with optimized parameters $f_{\theta_{t-1}^{'}}$ and corresponding graph view $(\mathbf{A}_{t-1}, \mathbf{X}_{t-1})$, i.e., $\mathbf{H}_{t-1}^{'}=f_{\theta_{t-1}^{'}}(\mathbf{A}_{t-1}, \mathbf{X}_{t-1})$, which can be easily obtained after backward and gradient updating at iteration $t-1$. For target encoder with optimized parameters at first iteration, we get them by random initialization. 
As for the specific architecture of graph encoder, we simply adopt the widely studied graph convolutional networks (GCN)~\cite{GCN} for a fair comparison. Note that the graph encoder can be arbitrarily specified here, such as GraphSAGE\cite{GraphSAGE}, GAT~\cite{GAT}.

\subsection{Inferential Predictor}
In order to obtain the prediction of target representation, we adopt an inferential method instead of a parameterized multilayer perceptron (MLP) according to theorem~\ref{theo:infer_predictor}, which results in quicker convergence and a more compact model with less parameters. To be more specific, the predictor can be formalized as
\begin{equation}\label{eq:predictor_infer}
    \mathbf{P}_{t}=\frac{\widebar{\mathbf{H}}_{t-1}^{'\top}\widebar{\mathbf{H}}_{t-1}^{'}}{N-1},
\end{equation}
where $\widebar{\mathbf{H}}_{t-1}^{'}$ is the target representations after zero-mean and $\ell_{2}$-normalization operations.
In specific, for each node $i$, the representation can be formalized as
\begin{equation}
    \widebar{\mathbf{H}}_{(t-1, i)}^{'} = \frac{\mathbf{H}_{(t-1, i)}^{'} - \textbf{m}}{||\mathbf{H}_{(t-1, i)}^{'} - \textbf{m}||_2},
\end{equation}
where $\textbf{m} = \frac{1}{N}\sum_{j=1}^{N}\mathbf{H}_{(t-1, j)}^{'}$.
Then we form the prediction $\mathbf{Z}_{t}$ from the corresponding online representation via the following formula,
\begin{equation}\label{eq: prediction}
    \mathbf{Z}_{t}=\mathbf{H}_{t}\mathbf{P}_{t}.
\end{equation}
 

Also note that though both online representation $\mathbf{H}_{t}$ and target representation $\mathbf{H}_{t-1}^{'}$ are optional to compute the covariance matrix under the assumption of theorem~\ref{theo:infer_predictor}, we empirically observe that the latter gives better performance and we will discuss it in  Section~\ref{sec:infer_predictor_ablation}.

\subsection{Objective Function}
Our objective function aims to maximize the cosine similarity between positive samples, i.e., the prediction $\mathbf{Z}_{(t, i)}$ of online representations produced from iteration $t$ and target representations $\mathbf{H}_{(t-1, i)}^{'}$ produced from iteration $t-1$
\begin{equation}\label{eq:loss}
    \mathcal{L}_{\theta} = 1 - \frac{1}{N}\sum_{i=1}^{N} \frac{\mathbf{Z}_{(t, i)}\mathbf{H}_{(t-1, i)}^{'\top}}{\vert\vert \mathbf{Z}_{(t, i)} \vert\vert_2 \vert\vert \mathbf{H}_{(t-1, i)}^{'} \vert\vert_2}.
\end{equation}

Despite Section~\ref{sec:thero_motivation} stating that the geometric direction alignment is the duty of loss function and both maximizing and minimizing the similarity is feasible, we pick the latter due to its relatively better performance and more acceptable meaning.

In practice, BGRL symmetrizes the loss function Eq.\eqref{eq:bgrl_loss} by predicting the target representation of the first augmented view using the online representation of the second. Here we do not follow the design since we only generate one single augmented view at each iteration and we empirically found the succinct design works well and computationally efficient. To help better understand the training process, we provide the detailed algorithm in Appendix~\ref{appdix:algorithm}.
When comes to the inference stage, we follow the previous literature~\cite{BGRL, CCA-SSG, GRACE} to feed the original graph $(\mathbf{A}, \mathbf{X})$ without any augmentations into the graph encoder $f_{\theta}(\cdot)$ to obtain final node representations $\mathbf{H}=f_{\theta}(\mathbf{A}, \mathbf{X})$ for various downstream benchmarks.

\subsection{Comparison With Previous Methods}
In this subsection, we systematically compare SGCL with previous graph contrastive learning frameworks, not limited to BGRL, including DGI~\cite{DGI}, MVGRL~\cite{MVGRL}, GRACE~\cite{GRACE}, GCA~\cite{GCA}, BGRL~\cite{BGRL}, AFGRL~\cite{AFGRL}, COSTA~\cite{COSTA}, CCA-SSG~\cite{CCA-SSG}. Table~\ref{tab:tech comparison} summarizes the technical differences between SGCL and previous methods. \\
\textbf{Negative pairs free}.
Previous methods typically rely on various negative pairs to maintain the uniformity property~\cite{algin_uniform} of the representations and avoid model collapse. For example, DGI and MVGRL obtain negative pairs through node feature shuffling and graph sub-sampling, respectively. GRACE considers the other nodes from two augmented views as negative samples. However, for the natural scarcity of negative samples and the complex connections between nodes in graphs, mining negative samples can be costly. BGRL and AFGRL attempt to address this issue by eliminating the need for negative pairs, but their asymmetric architecture and EMA design can be intricate and difficult to understand. CCA-SSG has turned to another possibility by decorrelating the feature dimensions based on Canonical Correlation Analysis~\cite{CCA}. However, CCA-SSG may not work well on datasets with low feature dimensions, as it essentially performs dimension reduction. In contrast, SGCL does not require any negative samples and has a concise design. \\
\textbf{One single encoder}.
To further improve model performance, most of the previous methods introduce additional modules. Both DGI and MVGRL devise an additional discriminator for mutual information estimation. GRACE and COSTA employ a projector to improve expressive ability and COSTA further provides additional optional encoders for multi-view learning. BGRL and AFGRL adopt the asymmetric architecture and EMA update mechanism to avoid model collapse. However, SGCL only needs a single encoder, which reduces model parameters and complexity. \\
\textbf{One single augmented view.}
For contrastive pairs construction, most previous methods produce two augmented graph views as the input of graph encoders at each iteration, except for DGI, COSTA and AFGRL. DGI performs augmentations once each iteration, yet it still needs to use an additional discriminator. COSTA allows for one augmentation in a single-view setting but at the expense of accuracy mostly. AFGRL does not require any augmentation, however, it requires KNN and K-means to hunt for local and global positive samples, both of which are even more time-consuming than graph augmentations. Nevertheless, SGCL only requires one single augmented view per iteration, which reduces the execution time and hyperparameter tuning time.

\section{Experiments}
In this section, we compare SGCL with state-of-the-art methods on eight public benchmarks. We report the averaged performance over twenty random dataset divisions and model initializations for all datasets apart from ten model initializations for ogbn-Arxiv, ogbn-MAG and ogbn-Products. For more experiments details, including dataset descriptions and implementation details, we refer readers of interest to Appendix \ref{appdix:exp_details}.

\begin{table*}[t]
    \centering
    \caption{Node classification accuracy ($\%$) on eight benchmark datasets. The \textbf{boldfaced} score denotes the best result. A.R.: average rank. OOM: out-of-memory on a 24GB RTX 3090Ti GPU.}
    \label{tab:node classification}
    \resizebox{1\linewidth}{!}{\begin{tabular}{l c c c c c c c c c c}
            \toprule
             & \textbf{Methods} & \textbf{WikiCS}       & \textbf{Amazon-Computers} & \textbf{Amazon-Photos} & \textbf{Coauthor-CS}  & \textbf{Coauthor-Physics} & \textbf{ogbn-Arxiv} & \textbf{ogbn-MAG} & \textbf{ogbn-Products} & \textbf{A.R.} $\downarrow$ \\
            \midrule
             & MLP              & 71.98 $\pm$ 0.00          & 73.81 $\pm$ 0.00              & 78.53 $\pm$ 0.00           & 90.37 $\pm$ 0.00          & 93.58 $\pm$ 0.00              & 56.30 $\pm$ 0.30          & 22.10 $\pm$ 0.30          & 61.06 $\pm$ 0.06          & 11.0                \\
             & GCN              & 77.19 $\pm$ 0.12          & 86.51 $\pm$ 0.54              & 92.42 $\pm$ 0.22           & 93.03 $\pm$ 0.31          & 95.65 $\pm$ 0.16              & 71.74 $\pm$ 0.29          & 30.10 $\pm$ 0.30          & 75.64 $\pm$ 0.21          & 6.6        \\
             & GAT              & 77.65 $\pm$ 0.11          & 86.93 $\pm$ 0.29              & 92.56 $\pm$ 0.35           & 92.31 $\pm$ 0.24          & 95.47 $\pm$ 0.15              & 70.60 $\pm$ 0.30          & 30.50 $\pm$ 0.30          & \textbf{79.45 $\pm$ 0.59}          & 6.3          \\
            \midrule
            
             & DGI              & 75.35 $\pm$ 0.14          & 83.95 $\pm$ 0.47              & 91.61 $\pm$ 0.22           & 92.15 $\pm$ 0.63          & 94.51 $\pm$ 0.52              & 65.10 $\pm$ 0.40          & OOM          & OOM          & 11.6                   \\
             & MVGRL            & 77.52 $\pm$ 0.08          & 87.52 $\pm$ 0.11              & 91.74 $\pm$ 0.07           & 92.11 $\pm$ 0.12          & 95.33 $\pm$ 0.03              & 68.10 $\pm$ 0.10          & OOM          & OOM          & 10.4                     \\
             & GRACE            & 77.97 $\pm$ 0.63          & 86.50 $\pm$ 0.33              & 92.46 $\pm$ 0.18           & 92.17 $\pm$ 0.04          & OOM                       & OOM          & OOM          & OOM                   & 10.6                   \\
             & GCA              & 78.35 $\pm$ 0.05          & 88.94 $\pm$ 0.15              & 92.53 $\pm$ 0.16           & 93.10 $\pm$ 0.01          & 95.73 $\pm$ 0.03              & 68.20 $\pm$ 0.20          & OOM          & OOM          & 6.6                     \\
             & COSTA            & 79.12 $\pm$ 0.02          & 88.32 $\pm$ 0.03              & 92.56 $\pm$ 0.45           & 92.95 $\pm$ 0.12          & 95.60 $\pm$ 0.02              & OOM          & OOM          & OOM                   & 7.8                     \\

             & AFGRL            & 77.62 $\pm$ 0.49          & 89.88 $\pm$ 0.33              & 93.22 $\pm$ 0.28           & 93.27 $\pm$ 0.17          & 95.69 $\pm$ 0.10              & OOM          & OOM          & OOM                   & 6.4                    \\
             & CCA-SSG          & 79.08 $\pm$ 0.53          & 88.74 $\pm$ 0.28              & 93.14 $\pm$ 0.14           & \textbf{93.32 $\pm$ 0.22}          & 95.38 $\pm$ 0.06              & 69.22 $\pm$ 0.22          & 31.78 $\pm$ 0.38          & 70.18 $\pm$ 0.15          & 5.3         \\
             & BGRL             & \textbf{79.98 $\pm$ 0.10} & 90.34 $\pm$ 0.19              & 93.17 $\pm$ 0.30           & 93.31 $\pm$ 0.13 & 95.73 $\pm$ 0.05              & 71.64 $\pm$ 0.12          & 32.18 $\pm$ 0.15          & 73.97 $\pm$ 0.05          & 2.8           \\
             & GraphMAE             & 79.92 $\pm$ 0.68 & 89.88 $\pm$ 0.10              & 93.41 $\pm$ 0.10           & 92.96 $\pm$ 0.09 & 95.40 $\pm$ 0.06              & \textbf{71.75 $\pm$ 0.17}          & 32.61 $\pm$ 0.11          & 70.23 $\pm$ 0.10          & 3.8          \\
             & SGCL        & 79.85 $\pm$ 0.53          & \textbf{90.70 $\pm$ 0.30}     & \textbf{93.46 $\pm$ 0.30}  & 93.29 $\pm$ 0.17          & \textbf{95.78 $\pm$ 0.11}     & 70.99 $\pm$ 0.09          & \textbf{32.71 $\pm$ 0.09}          & 75.96 $\pm$ 0.11         & \textbf{2.0}            \\
            \bottomrule
        \end{tabular}
    }
\end{table*}

\begin{table*}[h]
    \centering
    \small
    \caption{Comparison of the number of parameters (\#Paras), GPU memory (Mem) and execution time per epoch (Time) on a set of standard benchmark graphs. - indicates running out of memory on a 24GB RTX 3090Ti GPU. }
    \resizebox{\linewidth}{!}{\begin{tabular}{lccccccccccccccccccc}
            \toprule
                      & \multicolumn{3}{c}{\textbf{WikiCS}}
                      &                                     & \multicolumn{3}{c}{\textbf{Amazon-Computers}}
                      &                                     & \multicolumn{3}{c}{\textbf{Coauthor-Physics}}
                      &                                     & \multicolumn{3}{c}{\textbf{ogbn-Arxiv}}      
                      &                                     & \multicolumn{3}{c}{\textbf{ogbn-Products}}       \\
            \cmidrule{2-4}  \cmidrule{6-8} \cmidrule{10-12} \cmidrule{14-16} \cmidrule{18-20}
                      & \#Paras                             & Mem                                           & Time    &  & \#Paras & Mem    & Time    &  & \#Paras & Mem    & Time    &  & \#Paras & Mem     & Time   &  & \#Paras & Mem     & Time    \\
            \midrule
            GRACE     & 1.10M                               & 5.69GB                                        & 0.1549s &  & 1.57M   & 8.07GB & 0.1859s &  & -       & -      & -       &  & -       & -       & -       &  & -       & -       & -       \\
            AFGRL     & 4.82M                               & 6.40GB                                        & 2.2395s &  & 1.84M   & 4.88GB & 1.3542s &  & -       & -      & -       &  & -       & -       & -       &  & -       & -       & -       \\
            CCA-SSG   & 1.36M                               & 2.36GB                                        & 0.0344s &  & 0.66M   & 2.18GB & 0.0687s &  & 4.57M   & 7.04GB & 0.5000s &  & 0.33M   & 7.72 GB & 0.1770s &  & 0.09M   & 20.11 GB & 3.0537s \\
            BGRL      & 0.84M                               & 4.59GB                                        & 0.0552s &  & 0.59M   & 2.96GB & 0.0296s &  & 4.51M   & 6.87GB & 0.0864s &  & 0.46M   & 10.67GB & 0.2920s &  & 0.19M   & 20.88GB & 1.0597s\\
            GraphMAE      & 2.71M                               & 2.16GB                                        & 0.0336s &  & 3.67M   & 2.16GB & 0.0439s &  & 19.37M   & 11.98GB & 0.3142s &  & 3.42M   & 13.70GB & 0.3973s &  & 0.18M   & 18.67GB & 1.6947s\\
            SGCL & 0.29M                               & 1.42GB                                        & 0.0096s &  & 0.23M   & 1.36GB & 0.0084s &  & 2.19M   & 4.87GB & 0.0317s &  & 0.17M   & 5.11GB  & 0.0882s &  & 0.03M   & 11.11GB  & 0.3856s \\
            \bottomrule
        \end{tabular}
    }
    \label{tab:efficiency}
\end{table*}

\subsection{Experimental Analysis}
\subsubsection{Overall Performance}
We report the summarized node classification performance in Table \ref{tab:node classification}. As we can observe, SGCL matches the performance of supervised baselines on all datasets and outperforms previous state-of-the-art methods on 5 out of 8 datasets. Moreover, SGCL has competitive results on the other 3 datasets and gives the highest average ranking on all datasets compared to other baselines. In particular, in addition to  WikiCS, Coauthor-CS, and ogbn-Arxiv, SGCL outperforms the most powerful self-supervised baseline BGRL, whose reported results are obtained by double augmentation hyperparameters tuning, more complex model architecture and parameter updating mechanism. 
It is worth mentioning that the proposed method achieves significant performance improvements on the largest dataset ogbn-Products, which verifies the effectiveness of SGCL. We attribute this improvement to the fact that the inferential predictor can facilitate a better convergence of the graph encoder, which will be further elaborated in Section~\ref{sec:conv_speed}.

\subsubsection{Efficiency}
In Table \ref{tab:efficiency}, we compare the number of model parameters, GPU memory cost and execution time for each training iteration on three medium-scale datasets and two large-scale datasets.
The results are obtained with the official code and hyperparameters or the closest we could reach the reported performance.
Overall, our methods achieve the lowest number of parameters, training time and GPU memory cost all the time. 
Compared to AFGRL which adopts the same architecture as BGRL, we yield up to two orders of magnitude training speedup with only 1/16 parameters and 1/4 memory on WikiCS.
Compared to GraphMAE which relies on an encoder-decoder framework, our approach achieves competitive performance with about 1/10 parameters, 1/2 GPU memory cost and 5-10 times execution speedup. 
As for BGRL which is known for its scalability, we could further cut half of its memory and run three times faster on ogbn-Arxiv and ogbn-Products datasets.
The results demonstrate the effectiveness of our simple SGCL with even better (or competitive) performance.
Note that we perform subgraph sampling on ogbn-Products dataset, so the sample efficiency may affect the speed. With more efficient sampling implementation, the speed-up effect will be more obvious.
We kindly note that CCA-SSG gives relatively good memory usage since its official code adopts a memory-efficient framework DGL~\cite{dgl}, whereas we use PyG~\cite{pyg}.

\begin{figure}[t]
    \centering
	\subfloat[Coauthor-Physics]{\includegraphics[width=.5\columnwidth]{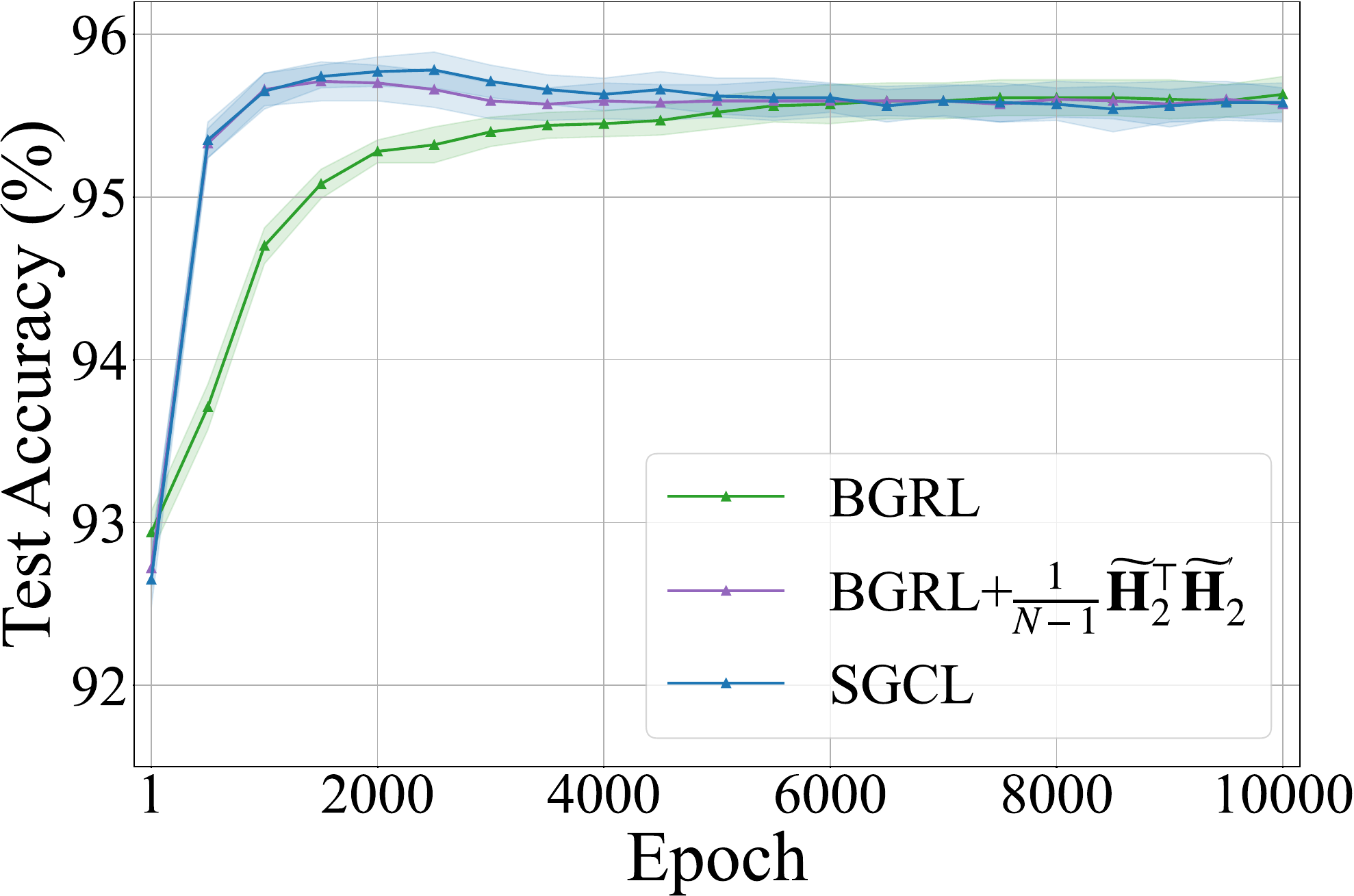}}
        \subfloat[ogbn-Products]{\includegraphics[width=.5\columnwidth]{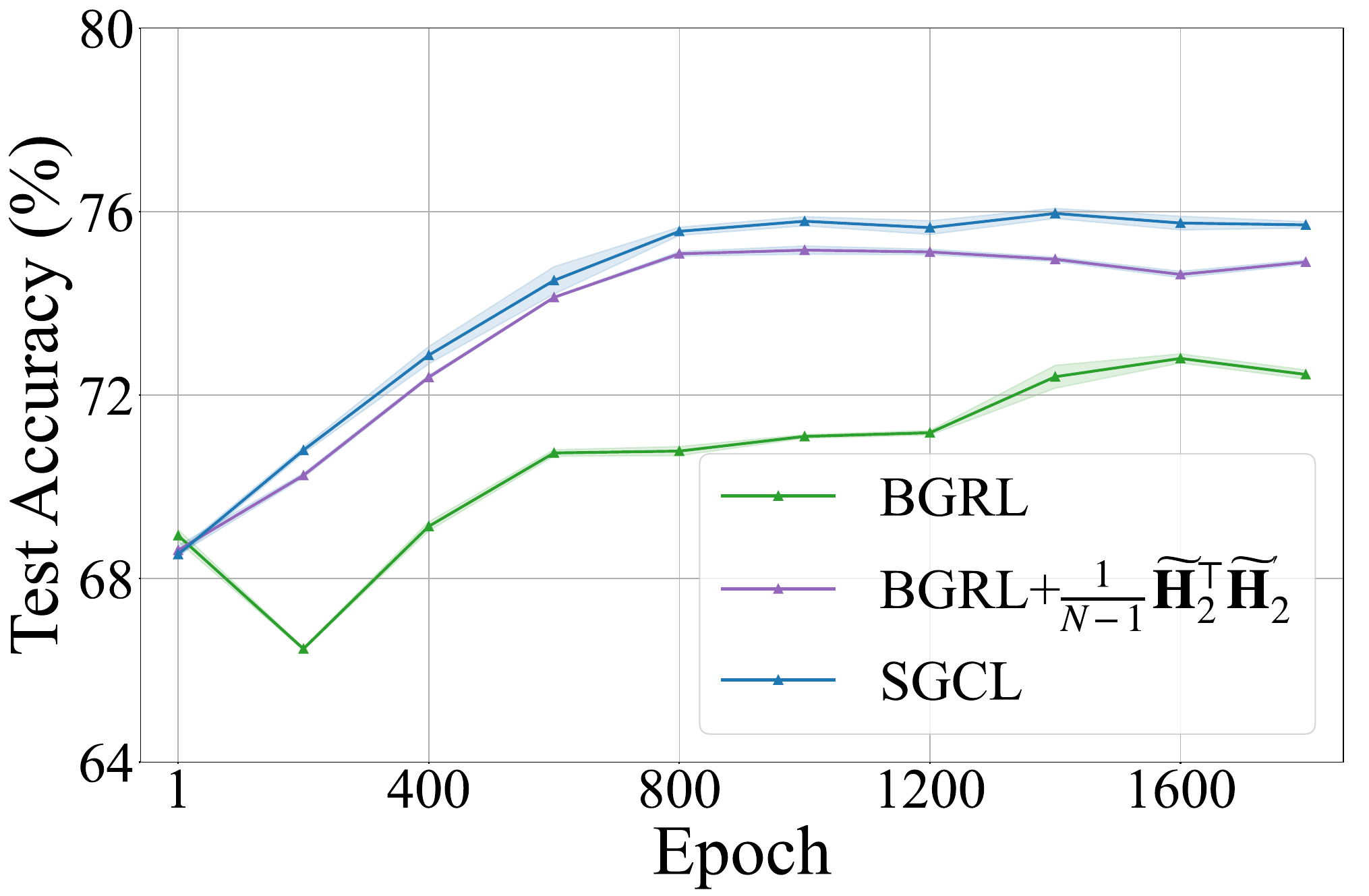}}
    \caption{Test accuracy curve of SGCL and BGRL.}
    \label{img:conv_photo_phy}
\end{figure}

\subsubsection{Convergence Speed}\label{sec:conv_speed}
In Figure \ref{img:conv_photo_phy}, we plot the test accuracy curve of BGRL and SGCL. We can see that SGCL consistently converges significantly faster than BGRL, especially in large-scale dataset ogbn-Products, which confirms the slow convergence issue inherent in BGRL and the superiority of our simplification. 
Furthermore, by replacing the parameterized predictor with an inferential predictor using the target representations $\widetilde{\mathbf{H}}_2$ in BGRL, the convergence speed is significantly improved, leading to a breakthrough in performance on ogbn-Products. 
This also serves as evidence that the slow convergence arises from the challenge of learning how to decorrelate using a parameterized predictor, whereas an inferential predictor analyzed in Section~\ref{sec:thero_motivation} can effectively mitigate this problem.
It is important to emphasize that our experiments can be at least an order of magnitude faster than BGRL with faster execution and convergence speed and fewer hyperparameters.
Note that we only include SGCL and BGRL in the comparison since they are the two strongest methods and can better validate our simplification.
In addition, we keep the encoder architecture setting consistent with BGRL for a fair comparison, i.e., the number of graph convolution layers and model dimensions.

\begin{table}[t]
    \caption{Ablation study for encoder simplification.  $\tau$: the decay rate of EMA, default to 0.99 in BGRL.}
    \label{tab:encoder_ablation_study}
    \scalebox{0.9}{
        \begin{tabular}{cccc}
            \toprule
             & \textbf{WikiCS} & \textbf{Amazon-Computers} & \textbf{Amazon-Photos} \\
            \midrule
            $\tau=0.99$     & 79.80 $\pm$ 0.47    & 90.13 $\pm$ 0.34              & 93.14 $\pm$ 0.28           \\
            $\tau=0.95$     & 79.89 $\pm$ 0.51    & 90.26 $\pm$ 0.27              & 94.41 $\pm$ 0.38           \\
            $\tau=0$        & 79.85 $\pm$ 0.55    & 90.67 $\pm$ 0.28              & 93.48 $\pm$ 0.34           \\
            SGCL       & 79.85 $\pm$ 0.53    & 90.70 $\pm$ 0.30              & 93.46 $\pm$ 0.30 \\
            \bottomrule
        \end{tabular}}
\end{table}

\begin{table}[t]
    \caption{Ablation study for inferential predictor. $\mathbf{I}$: removing predictor. ${\frac{1}{N-1}}{\protect\widebar{\mathbf{H}}^{\top}_{t}}{\protect\widebar{\mathbf{H}}_{t}}$: setting predictor to the covariance matrix of representations from current iteration. MLP: setting predictor to a MLP following BGRL. BGRL+$\frac{1}{N-1}\widetilde{\mathbf{H}}_2^{\top}\widetilde{\mathbf{H}}_2^{'}$: replacing original MLP predictor to $\frac{1}{N-1}\widetilde{\mathbf{H}}_2^{\top}\widetilde{\mathbf{H}}_2$ for BGRL.}
    \label{tab:predictor_ablation_study}
    \scalebox{0.8}{
        \begin{tabular}{cccc}
            \toprule
            & \textbf{WikiCS} & \textbf{Amazon-Computers} & \textbf{Amazon-Photos} \\
            \midrule
                       $\frac{1}{N-1}\widebar{\mathbf{H}}^{'\top}_{t-1}\widebar{\mathbf{H}}^{'}_{t-1}$
                             & 79.85 $\pm$ 0.53    & 90.70 $\pm$ 0.30              & 93.46 $\pm$ 0.30           \\
            $\mathbf{I}$     & 78.52 $\pm$ 0.48    & 84.73 $\pm$ 0.38              & 91.92 $\pm$ 0.42           \\
            MLP              & 79.34 $\pm$ 0.52    & 88.70 $\pm$ 0.31              & 92.98 $\pm$ 0.33           \\
            $\frac{1}{N-1}\widebar{\mathbf{H}}^{\top}_{t}\widebar{\mathbf{H}}_{t}$
                             & 78.64 $\pm$ 0.54    & 90.45 $\pm$ 0.29              & 93.35 $\pm$ 0.31           \\
            BGRL+$\frac{1}{N-1}\widetilde{\mathbf{H}}_2^{\top}\widetilde{\mathbf{H}}_2^{'}$
                             & 79.85 $\pm$ 0.45    & 90.26 $\pm$ 0.33              & 93.12 $\pm$ 0.33           \\
            \bottomrule
        \end{tabular}}
\end{table}

\subsection{Ablation Studies}

To verify the effectiveness of encoder simplification and inferential predictor, we conduct ablation experiments and report corresponding node classification performance.
We keep all the other hyperparameters consistent throughout the experiments. 

\subsubsection{Effect of the encoder simplification.} 
To study the influence of the encoder simplification, we compare the difference between using a single encoder and an additional target encoder with an EMA parameter updating mechanism as BGRL. As Table~\ref{tab:encoder_ablation_study} shows, the performance of using a single encoder in our proposed framework is basically the same as using an additional target encoder with various EMA decay rates, which proves the validity of our simplification. Also, the similar performance under different decay rates is consistent with the non-necessity of EMA as stated in Section~\ref{sec:emp_motivation}.
In fact, we still implicitly take the optimized online encoder from the previous iteration as the target encoder, which is equivalent to $\tau=0$ and they do give the most proximate performance. 

\subsubsection{Effect of the inferential predictor.} \label{sec:infer_predictor_ablation}
Table~\ref{tab:predictor_ablation_study} shows the influence of the inferential predictor. First, removing the inferential predictor will lead to a significant performance drop, which illustrates the effectiveness of the inferential predictor.
Second, resetting the predictor to MLP as BGRL gives similar results as our method, which proves the correctness of our inference. 
Moreover, we observe that setting the predictor as the covariance matrix of $\widebar{\mathbf{H}}_{t}$ gives slightly worse performance than $\widebar{\mathbf{H}}^{'}_{t-1}$, this can be explained that $\widebar{\mathbf{H}}^{'}_{t-1}$ is obtained from the optimized parameters for the corresponding training input $(\mathbf{A}_{t-1}, \mathbf{X}_{t-1})$. Thus, $\widebar{\mathbf{H}}^{'}_{t-1}$ is more stable and accurate to serve as the target to help the model learn better.
Finally, substituting the original MLP predictor in BGRL with the inferred predictor yielded similar performance compared to vanilla BGRL, thereby confirming the validity of our theoretical analysis in Section~\ref{sec:thero_motivation} and the effectiveness of the inferential predictor.

\begin{figure}[t]
    \centering
    \small
    \includegraphics[width=1\linewidth]{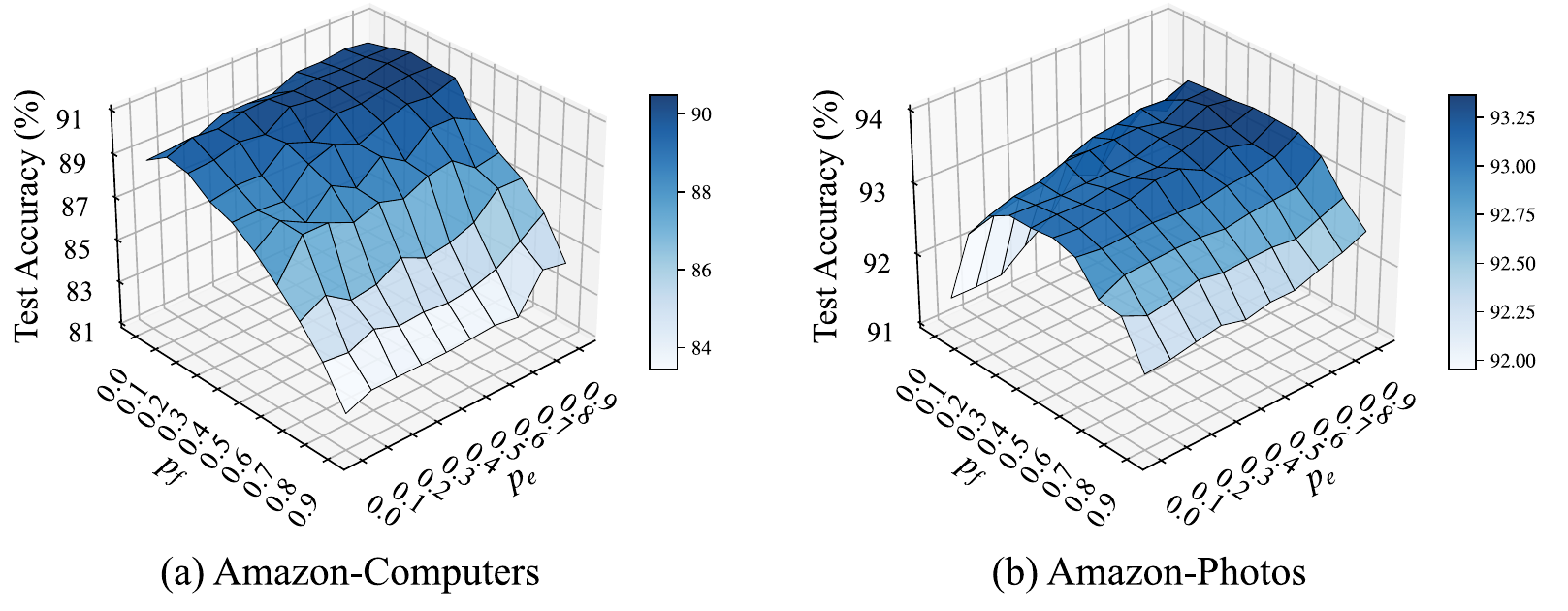}
    \caption{Effect of $p_{e}$ and $p_{f}$.}
    \label{img:drop_p_abla}
\end{figure}

\subsection{Hyperparameter Analysis}
We investigate the impact of graph augmentation hyperparameters in SGCL, i.e., edge drop ratio $p_{e}$ and feature drop ratio $p_{f}$. We keep the other parameters the same while only changing $p_{e}$ and $p_{f}$. We conduct experiments by varying the values of $p_{e}$ and $p_{f}$ from 0 to 0.9 and report the corresponding test accuracy in Figure \ref{img:drop_p_abla}. From the figure, we can observe that the classification accuracy is generally stable. That is, as long as the augmentation parameters are in a proper range, SGCL could consistently achieve competitive performance.
However, applying an appropriate graph augmentation can effectively improve the model performance, which is a further validation of our analysis.
In addition, we find that SGCL benefits from a larger edge drop ratio and we attribute the fact that we do not apply distinct augmentation functions like previous GCL methods and therefore need a greater degree of perturbation to produce more discriminating augmented views.

\section{Conclusion}
In this paper, we empirically show that the graph augmentations and the predictor are crucial to the success of BGRL and give our insights on their role in the framework. 
We theoretically demonstrate that the predictor could be computed from node representations. Through our empirical and theoretical analysis, we have uncovered potential redundancies in BGRL and aim to simplify the framework accordingly. We propose a simple yet efficient negative-sample-free GCL framework SGCL, which only contains a graph augmentation, a graph encoder and an inferential predictor without any other parameters. 
Extensive experiments on eight benchmarks demonstrate the effectiveness of SGCL, which achieves competitive performance with BGRL and state-of-the-arts while effectively reducing the number of parameters and memory consumption and accelerating the execution and convergence speed.

\section{ACKNOWLEDGEMENTS}
The research is supported by the National Key R\&D Program of China under grant No. 2022YFF0902500, the Guangdong Basic and Applied Basic Research Foundation, China (No. 2023A1515011050).

\bibliographystyle{ACM-Reference-Format}
\bibliography{main}


\begin{thebibliography}{40}


\ifx \showCODEN    \undefined \def \showCODEN     #1{\unskip}     \fi
\ifx \showDOI      \undefined \def \showDOI       #1{#1}\fi
\ifx \showISBNx    \undefined \def \showISBNx     #1{\unskip}     \fi
\ifx \showISBNxiii \undefined \def \showISBNxiii  #1{\unskip}     \fi
\ifx \showISSN     \undefined \def \showISSN      #1{\unskip}     \fi
\ifx \showLCCN     \undefined \def \showLCCN      #1{\unskip}     \fi
\ifx \shownote     \undefined \def \shownote      #1{#1}          \fi
\ifx \showarticletitle \undefined \def \showarticletitle #1{#1}   \fi
\ifx \showURL      \undefined \def \showURL       {\relax}        \fi
\providecommand\bibfield[2]{#2}
\providecommand\bibinfo[2]{#2}
\providecommand\natexlab[1]{#1}
\providecommand\showeprint[2][]{arXiv:#2}

\bibitem[Ba et~al\mbox{.}(2016)]%
        {ba2016layer}
\bibfield{author}{\bibinfo{person}{Jimmy~Lei Ba}, \bibinfo{person}{Jamie~Ryan Kiros}, {and} \bibinfo{person}{Geoffrey~E Hinton}.} \bibinfo{year}{2016}\natexlab{}.
\newblock \showarticletitle{Layer normalization}.
\newblock \bibinfo{journal}{\emph{arXiv preprint arXiv:1607.06450}} (\bibinfo{year}{2016}).
\newblock


\bibitem[Chen et~al\mbox{.}(2020)]%
        {SimCLR}
\bibfield{author}{\bibinfo{person}{Ting Chen}, \bibinfo{person}{Simon Kornblith}, \bibinfo{person}{Mohammad Norouzi}, {and} \bibinfo{person}{Geoffrey~E. Hinton}.} \bibinfo{year}{2020}\natexlab{}.
\newblock \showarticletitle{A Simple Framework for Contrastive Learning of Visual Representations}. In \bibinfo{booktitle}{\emph{Proceedings of the 37th International Conference on Machine Learning, {ICML} 2020, 13-18 July 2020, Virtual Event}} \emph{(\bibinfo{series}{Proceedings of Machine Learning Research}, Vol.~\bibinfo{volume}{119})}. \bibinfo{publisher}{{PMLR}}, \bibinfo{pages}{1597--1607}.
\newblock
\urldef\tempurl%
\url{http://proceedings.mlr.press/v119/chen20j.html}
\showURL{%
\tempurl}


\bibitem[Chen and He(2021)]%
        {Siamese}
\bibfield{author}{\bibinfo{person}{Xinlei Chen} {and} \bibinfo{person}{Kaiming He}.} \bibinfo{year}{2021}\natexlab{}.
\newblock \showarticletitle{Exploring simple siamese representation learning}. In \bibinfo{booktitle}{\emph{Proceedings of the IEEE/CVF Conference on Computer Vision and Pattern Recognition}}. \bibinfo{pages}{15750--15758}.
\newblock


\bibitem[Fey and Lenssen(2019)]%
        {pyg}
\bibfield{author}{\bibinfo{person}{Matthias Fey} {and} \bibinfo{person}{Jan~E. Lenssen}.} \bibinfo{year}{2019}\natexlab{}.
\newblock \showarticletitle{Fast Graph Representation Learning with {PyTorch Geometric}}. In \bibinfo{booktitle}{\emph{ICLR Workshop on Representation Learning on Graphs and Manifolds}}.
\newblock


\bibitem[Glorot and Bengio(2010)]%
        {Glorot}
\bibfield{author}{\bibinfo{person}{Xavier Glorot} {and} \bibinfo{person}{Yoshua Bengio}.} \bibinfo{year}{2010}\natexlab{}.
\newblock \showarticletitle{Understanding the difficulty of training deep feedforward neural networks}. In \bibinfo{booktitle}{\emph{Proceedings of the Thirteenth International Conference on Artificial Intelligence and Statistics, {AISTATS} 2010, Chia Laguna Resort, Sardinia, Italy, May 13-15, 2010}} \emph{(\bibinfo{series}{{JMLR} Proceedings}, Vol.~\bibinfo{volume}{9})}, \bibfield{editor}{\bibinfo{person}{Yee~Whye Teh} {and} \bibinfo{person}{D.~Mike Titterington}} (Eds.). \bibinfo{publisher}{JMLR.org}, \bibinfo{pages}{249--256}.
\newblock
\urldef\tempurl%
\url{http://proceedings.mlr.press/v9/glorot10a.html}
\showURL{%
\tempurl}


\bibitem[Grill et~al\mbox{.}(2020)]%
        {BYOL}
\bibfield{author}{\bibinfo{person}{Jean-Bastien Grill}, \bibinfo{person}{Florian Strub}, \bibinfo{person}{Florent Altch{\'e}}, \bibinfo{person}{Corentin Tallec}, \bibinfo{person}{Pierre Richemond}, \bibinfo{person}{Elena Buchatskaya}, \bibinfo{person}{Carl Doersch}, \bibinfo{person}{Bernardo Avila~Pires}, \bibinfo{person}{Zhaohan Guo}, \bibinfo{person}{Mohammad Gheshlaghi~Azar}, {et~al\mbox{.}}} \bibinfo{year}{2020}\natexlab{}.
\newblock \showarticletitle{Bootstrap your own latent-a new approach to self-supervised learning}.
\newblock \bibinfo{journal}{\emph{Advances in neural information processing systems}}  \bibinfo{volume}{33} (\bibinfo{year}{2020}), \bibinfo{pages}{21271--21284}.
\newblock


\bibitem[Hamilton et~al\mbox{.}(2017)]%
        {GraphSAGE}
\bibfield{author}{\bibinfo{person}{Will Hamilton}, \bibinfo{person}{Zhitao Ying}, {and} \bibinfo{person}{Jure Leskovec}.} \bibinfo{year}{2017}\natexlab{}.
\newblock \showarticletitle{Inductive representation learning on large graphs}.
\newblock \bibinfo{journal}{\emph{Advances in neural information processing systems}}  \bibinfo{volume}{30} (\bibinfo{year}{2017}).
\newblock


\bibitem[Hassani and Khasahmadi(2020)]%
        {MVGRL}
\bibfield{author}{\bibinfo{person}{Kaveh Hassani} {and} \bibinfo{person}{Amir~Hosein Khasahmadi}.} \bibinfo{year}{2020}\natexlab{}.
\newblock \showarticletitle{Contrastive multi-view representation learning on graphs}. In \bibinfo{booktitle}{\emph{International Conference on Machine Learning}}. PMLR, \bibinfo{pages}{4116--4126}.
\newblock


\bibitem[He et~al\mbox{.}(2020b)]%
        {MoCo}
\bibfield{author}{\bibinfo{person}{Kaiming He}, \bibinfo{person}{Haoqi Fan}, \bibinfo{person}{Yuxin Wu}, \bibinfo{person}{Saining Xie}, {and} \bibinfo{person}{Ross~B. Girshick}.} \bibinfo{year}{2020}\natexlab{b}.
\newblock \showarticletitle{Momentum Contrast for Unsupervised Visual Representation Learning}. In \bibinfo{booktitle}{\emph{2020 {IEEE/CVF} Conference on Computer Vision and Pattern Recognition, {CVPR} 2020, Seattle, WA, USA, June 13-19, 2020}}. \bibinfo{publisher}{Computer Vision Foundation / {IEEE}}, \bibinfo{pages}{9726--9735}.
\newblock
\urldef\tempurl%
\url{https://doi.org/10.1109/CVPR42600.2020.00975}
\showDOI{\tempurl}


\bibitem[He et~al\mbox{.}(2020a)]%
        {LightGCN}
\bibfield{author}{\bibinfo{person}{Xiangnan He}, \bibinfo{person}{Kuan Deng}, \bibinfo{person}{Xiang Wang}, \bibinfo{person}{Yan Li}, \bibinfo{person}{Yong{-}Dong Zhang}, {and} \bibinfo{person}{Meng Wang}.} \bibinfo{year}{2020}\natexlab{a}.
\newblock \showarticletitle{LightGCN: Simplifying and Powering Graph Convolution Network for Recommendation}. In \bibinfo{booktitle}{\emph{Proceedings of the 43rd International {ACM} {SIGIR} conference on research and development in Information Retrieval, {SIGIR} 2020, Virtual Event, China, July 25-30, 2020}}, \bibfield{editor}{\bibinfo{person}{Jimmy~X. Huang}, \bibinfo{person}{Yi~Chang}, \bibinfo{person}{Xueqi Cheng}, \bibinfo{person}{Jaap Kamps}, \bibinfo{person}{Vanessa Murdock}, \bibinfo{person}{Ji{-}Rong Wen}, {and} \bibinfo{person}{Yiqun Liu}} (Eds.). \bibinfo{publisher}{{ACM}}, \bibinfo{pages}{639--648}.
\newblock
\urldef\tempurl%
\url{https://doi.org/10.1145/3397271.3401063}
\showDOI{\tempurl}


\bibitem[Hotelling(1992)]%
        {CCA}
\bibfield{author}{\bibinfo{person}{Harold Hotelling}.} \bibinfo{year}{1992}\natexlab{}.
\newblock \showarticletitle{Relations between two sets of variates}.
\newblock In \bibinfo{booktitle}{\emph{Breakthroughs in statistics}}. \bibinfo{publisher}{Springer}, \bibinfo{pages}{162--190}.
\newblock


\bibitem[Hou et~al\mbox{.}(2022)]%
        {GraphMAE}
\bibfield{author}{\bibinfo{person}{Zhenyu Hou}, \bibinfo{person}{Xiao Liu}, \bibinfo{person}{Yukuo Cen}, \bibinfo{person}{Yuxiao Dong}, \bibinfo{person}{Hongxia Yang}, \bibinfo{person}{Chunjie Wang}, {and} \bibinfo{person}{Jie Tang}.} \bibinfo{year}{2022}\natexlab{}.
\newblock \showarticletitle{GraphMAE: Self-Supervised Masked Graph Autoencoders}. In \bibinfo{booktitle}{\emph{{KDD} '22: The 28th {ACM} {SIGKDD} Conference on Knowledge Discovery and Data Mining, Washington, DC, USA, August 14 - 18, 2022}}, \bibfield{editor}{\bibinfo{person}{Aidong Zhang} {and} \bibinfo{person}{Huzefa Rangwala}} (Eds.). \bibinfo{publisher}{{ACM}}, \bibinfo{pages}{594--604}.
\newblock
\urldef\tempurl%
\url{https://doi.org/10.1145/3534678.3539321}
\showDOI{\tempurl}


\bibitem[Hu et~al\mbox{.}(2020)]%
        {hu2020ogb}
\bibfield{author}{\bibinfo{person}{Weihua Hu}, \bibinfo{person}{Matthias Fey}, \bibinfo{person}{Marinka Zitnik}, \bibinfo{person}{Yuxiao Dong}, \bibinfo{person}{Hongyu Ren}, \bibinfo{person}{Bowen Liu}, \bibinfo{person}{Michele Catasta}, {and} \bibinfo{person}{Jure Leskovec}.} \bibinfo{year}{2020}\natexlab{}.
\newblock \showarticletitle{Open Graph Benchmark: Datasets for Machine Learning on Graphs}.
\newblock \bibinfo{journal}{\emph{arXiv preprint arXiv:2005.00687}} (\bibinfo{year}{2020}).
\newblock


\bibitem[Ioffe and Szegedy(2015)]%
        {ioffe2015batch}
\bibfield{author}{\bibinfo{person}{Sergey Ioffe} {and} \bibinfo{person}{Christian Szegedy}.} \bibinfo{year}{2015}\natexlab{}.
\newblock \showarticletitle{Batch normalization: Accelerating deep network training by reducing internal covariate shift}. In \bibinfo{booktitle}{\emph{International conference on machine learning}}. PMLR, \bibinfo{pages}{448--456}.
\newblock


\bibitem[Kingma and Ba(2014)]%
        {Adam}
\bibfield{author}{\bibinfo{person}{Diederik~P Kingma} {and} \bibinfo{person}{Jimmy Ba}.} \bibinfo{year}{2014}\natexlab{}.
\newblock \showarticletitle{Adam: A method for stochastic optimization}.
\newblock \bibinfo{journal}{\emph{arXiv preprint arXiv:1412.6980}} (\bibinfo{year}{2014}).
\newblock


\bibitem[Kipf and Welling(2016)]%
        {GCN}
\bibfield{author}{\bibinfo{person}{Thomas~N Kipf} {and} \bibinfo{person}{Max Welling}.} \bibinfo{year}{2016}\natexlab{}.
\newblock \showarticletitle{Semi-supervised classification with graph convolutional networks}.
\newblock \bibinfo{journal}{\emph{arXiv preprint arXiv:1609.02907}} (\bibinfo{year}{2016}).
\newblock


\bibitem[Lampinen and Ganguli(2019)]%
        {TS_Dyna}
\bibfield{author}{\bibinfo{person}{Andrew~K. Lampinen} {and} \bibinfo{person}{Surya Ganguli}.} \bibinfo{year}{2019}\natexlab{}.
\newblock \showarticletitle{An analytic theory of generalization dynamics and transfer learning in deep linear networks}. In \bibinfo{booktitle}{\emph{7th International Conference on Learning Representations, {ICLR} 2019, New Orleans, LA, USA, May 6-9, 2019}}. \bibinfo{publisher}{OpenReview.net}.
\newblock
\urldef\tempurl%
\url{https://openreview.net/forum?id=ryfMLoCqtQ}
\showURL{%
\tempurl}


\bibitem[Lee et~al\mbox{.}(2022)]%
        {AFGRL}
\bibfield{author}{\bibinfo{person}{Namkyeong Lee}, \bibinfo{person}{Junseok Lee}, {and} \bibinfo{person}{Chanyoung Park}.} \bibinfo{year}{2022}\natexlab{}.
\newblock \showarticletitle{Augmentation-free self-supervised learning on graphs}. In \bibinfo{booktitle}{\emph{Proceedings of the AAAI Conference on Artificial Intelligence}}, Vol.~\bibinfo{volume}{36}. \bibinfo{pages}{7372--7380}.
\newblock


\bibitem[Liu et~al\mbox{.}(2022)]%
        {gssl_survey}
\bibfield{author}{\bibinfo{person}{Yixin Liu}, \bibinfo{person}{Ming Jin}, \bibinfo{person}{Shirui Pan}, \bibinfo{person}{Chuan Zhou}, \bibinfo{person}{Yu Zheng}, \bibinfo{person}{Feng Xia}, {and} \bibinfo{person}{Philip Yu}.} \bibinfo{year}{2022}\natexlab{}.
\newblock \showarticletitle{Graph self-supervised learning: A survey}.
\newblock \bibinfo{journal}{\emph{IEEE Transactions on Knowledge and Data Engineering}} (\bibinfo{year}{2022}).
\newblock


\bibitem[Loshchilov and Hutter(2017)]%
        {loshchilov2017decoupled}
\bibfield{author}{\bibinfo{person}{Ilya Loshchilov} {and} \bibinfo{person}{Frank Hutter}.} \bibinfo{year}{2017}\natexlab{}.
\newblock \showarticletitle{Decoupled weight decay regularization}.
\newblock \bibinfo{journal}{\emph{arXiv preprint arXiv:1711.05101}} (\bibinfo{year}{2017}).
\newblock


\bibitem[Mernyei and Cangea(2020)]%
        {mernyei2020wiki}
\bibfield{author}{\bibinfo{person}{P{\'e}ter Mernyei} {and} \bibinfo{person}{C{\u{a}}t{\u{a}}lina Cangea}.} \bibinfo{year}{2020}\natexlab{}.
\newblock \showarticletitle{Wiki-cs: A wikipedia-based benchmark for graph neural networks}.
\newblock \bibinfo{journal}{\emph{arXiv preprint arXiv:2007.02901}} (\bibinfo{year}{2020}).
\newblock


\bibitem[Shchur et~al\mbox{.}(2018)]%
        {shchur2018pitfalls}
\bibfield{author}{\bibinfo{person}{Oleksandr Shchur}, \bibinfo{person}{Maximilian Mumme}, \bibinfo{person}{Aleksandar Bojchevski}, {and} \bibinfo{person}{Stephan G{\"u}nnemann}.} \bibinfo{year}{2018}\natexlab{}.
\newblock \showarticletitle{Pitfalls of Graph Neural Network Evaluation}.
\newblock \bibinfo{journal}{\emph{Relational Representation Learning Workshop, NeurIPS}} (\bibinfo{year}{2018}).
\newblock


\bibitem[Thakoor et~al\mbox{.}(2022)]%
        {BGRL}
\bibfield{author}{\bibinfo{person}{Shantanu Thakoor}, \bibinfo{person}{Corentin Tallec}, \bibinfo{person}{Mohammad~Gheshlaghi Azar}, \bibinfo{person}{Mehdi Azabou}, \bibinfo{person}{Eva~L Dyer}, \bibinfo{person}{Remi Munos}, \bibinfo{person}{Petar Veli{\v{c}}kovi{\'c}}, {and} \bibinfo{person}{Michal Valko}.} \bibinfo{year}{2022}\natexlab{}.
\newblock \showarticletitle{Large-Scale Representation Learning on Graphs via Bootstrapping}. In \bibinfo{booktitle}{\emph{International Conference on Learning Representations}}.
\newblock
\urldef\tempurl%
\url{https://openreview.net/forum?id=0UXT6PpRpW}
\showURL{%
\tempurl}


\bibitem[Veli{\v{c}}kovi{\'c} et~al\mbox{.}(2017)]%
        {GAT}
\bibfield{author}{\bibinfo{person}{Petar Veli{\v{c}}kovi{\'c}}, \bibinfo{person}{Guillem Cucurull}, \bibinfo{person}{Arantxa Casanova}, \bibinfo{person}{Adriana Romero}, \bibinfo{person}{Pietro Lio}, {and} \bibinfo{person}{Yoshua Bengio}.} \bibinfo{year}{2017}\natexlab{}.
\newblock \showarticletitle{Graph attention networks}.
\newblock \bibinfo{journal}{\emph{arXiv preprint arXiv:1710.10903}} (\bibinfo{year}{2017}).
\newblock


\bibitem[Velickovic et~al\mbox{.}(2019)]%
        {DGI}
\bibfield{author}{\bibinfo{person}{Petar Velickovic}, \bibinfo{person}{William Fedus}, \bibinfo{person}{William~L Hamilton}, \bibinfo{person}{Pietro Li{\`o}}, \bibinfo{person}{Yoshua Bengio}, {and} \bibinfo{person}{R~Devon Hjelm}.} \bibinfo{year}{2019}\natexlab{}.
\newblock \showarticletitle{Deep Graph Infomax.}
\newblock \bibinfo{journal}{\emph{ICLR (Poster)}} \bibinfo{volume}{2}, \bibinfo{number}{3} (\bibinfo{year}{2019}), \bibinfo{pages}{4}.
\newblock


\bibitem[Wang et~al\mbox{.}(2019)]%
        {dgl}
\bibfield{author}{\bibinfo{person}{Minjie Wang}, \bibinfo{person}{Da Zheng}, \bibinfo{person}{Zihao Ye}, \bibinfo{person}{Quan Gan}, \bibinfo{person}{Mufei Li}, \bibinfo{person}{Xiang Song}, \bibinfo{person}{Jinjing Zhou}, \bibinfo{person}{Chao Ma}, \bibinfo{person}{Lingfan Yu}, \bibinfo{person}{Yu Gai}, \bibinfo{person}{Tianjun Xiao}, \bibinfo{person}{Tong He}, \bibinfo{person}{George Karypis}, \bibinfo{person}{Jinyang Li}, {and} \bibinfo{person}{Zheng Zhang}.} \bibinfo{year}{2019}\natexlab{}.
\newblock \showarticletitle{Deep Graph Library: A Graph-Centric, Highly-Performant Package for Graph Neural Networks}.
\newblock \bibinfo{journal}{\emph{arXiv preprint arXiv:1909.01315}} (\bibinfo{year}{2019}).
\newblock


\bibitem[Wang and Isola(2020)]%
        {algin_uniform}
\bibfield{author}{\bibinfo{person}{Tongzhou Wang} {and} \bibinfo{person}{Phillip Isola}.} \bibinfo{year}{2020}\natexlab{}.
\newblock \showarticletitle{Understanding contrastive representation learning through alignment and uniformity on the hypersphere}. In \bibinfo{booktitle}{\emph{International Conference on Machine Learning}}. PMLR, \bibinfo{pages}{9929--9939}.
\newblock


\bibitem[Wang et~al\mbox{.}(2022)]%
        {MoICLR}
\bibfield{author}{\bibinfo{person}{Yuyang Wang}, \bibinfo{person}{Jianren Wang}, \bibinfo{person}{Zhonglin Cao}, {and} \bibinfo{person}{Amir~Barati Farimani}.} \bibinfo{year}{2022}\natexlab{}.
\newblock \showarticletitle{Molecular contrastive learning of representations via graph neural networks}.
\newblock \bibinfo{journal}{\emph{Nat. Mach. Intell.}} \bibinfo{volume}{4}, \bibinfo{number}{3} (\bibinfo{year}{2022}), \bibinfo{pages}{279--287}.
\newblock
\urldef\tempurl%
\url{https://doi.org/10.1038/s42256-022-00447-x}
\showDOI{\tempurl}


\bibitem[Wu et~al\mbox{.}(2019)]%
        {SGC}
\bibfield{author}{\bibinfo{person}{Felix Wu}, \bibinfo{person}{Amauri Souza}, \bibinfo{person}{Tianyi Zhang}, \bibinfo{person}{Christopher Fifty}, \bibinfo{person}{Tao Yu}, {and} \bibinfo{person}{Kilian Weinberger}.} \bibinfo{year}{2019}\natexlab{}.
\newblock \showarticletitle{Simplifying graph convolutional networks}. In \bibinfo{booktitle}{\emph{International conference on machine learning}}. PMLR, \bibinfo{pages}{6861--6871}.
\newblock


\bibitem[Wu et~al\mbox{.}(2021b)]%
        {SGL}
\bibfield{author}{\bibinfo{person}{Jiancan Wu}, \bibinfo{person}{Xiang Wang}, \bibinfo{person}{Fuli Feng}, \bibinfo{person}{Xiangnan He}, \bibinfo{person}{Liang Chen}, \bibinfo{person}{Jianxun Lian}, {and} \bibinfo{person}{Xing Xie}.} \bibinfo{year}{2021}\natexlab{b}.
\newblock \showarticletitle{Self-supervised Graph Learning for Recommendation}. In \bibinfo{booktitle}{\emph{{SIGIR} '21: The 44th International {ACM} {SIGIR} Conference on Research and Development in Information Retrieval, Virtual Event, Canada, July 11-15, 2021}}, \bibfield{editor}{\bibinfo{person}{Fernando Diaz}, \bibinfo{person}{Chirag Shah}, \bibinfo{person}{Torsten Suel}, \bibinfo{person}{Pablo Castells}, \bibinfo{person}{Rosie Jones}, {and} \bibinfo{person}{Tetsuya Sakai}} (Eds.). \bibinfo{publisher}{{ACM}}, \bibinfo{pages}{726--735}.
\newblock
\urldef\tempurl%
\url{https://doi.org/10.1145/3404835.3462862}
\showDOI{\tempurl}


\bibitem[Wu et~al\mbox{.}(2021a)]%
        {GSSL_P_C_G}
\bibfield{author}{\bibinfo{person}{Lirong Wu}, \bibinfo{person}{Haitao Lin}, \bibinfo{person}{Zhangyang Gao}, \bibinfo{person}{Cheng Tan}, {and} \bibinfo{person}{Stan~Z. Li}.} \bibinfo{year}{2021}\natexlab{a}.
\newblock \showarticletitle{Self-supervised on Graphs: Contrastive, Generative, or Predictive}.
\newblock \bibinfo{journal}{\emph{CoRR}}  \bibinfo{volume}{abs/2105.07342} (\bibinfo{year}{2021}).
\newblock
\showeprint[arXiv]{2105.07342}
\urldef\tempurl%
\url{https://arxiv.org/abs/2105.07342}
\showURL{%
\tempurl}


\bibitem[Xia et~al\mbox{.}(2021)]%
        {SSL_Hyperrec}
\bibfield{author}{\bibinfo{person}{Xin Xia}, \bibinfo{person}{Hongzhi Yin}, \bibinfo{person}{Junliang Yu}, \bibinfo{person}{Qinyong Wang}, \bibinfo{person}{Lizhen Cui}, {and} \bibinfo{person}{Xiangliang Zhang}.} \bibinfo{year}{2021}\natexlab{}.
\newblock \showarticletitle{Self-Supervised Hypergraph Convolutional Networks for Session-based Recommendation}. In \bibinfo{booktitle}{\emph{Thirty-Fifth {AAAI} Conference on Artificial Intelligence, {AAAI} 2021, Thirty-Third Conference on Innovative Applications of Artificial Intelligence, {IAAI} 2021, The Eleventh Symposium on Educational Advances in Artificial Intelligence, {EAAI} 2021, Virtual Event, February 2-9, 2021}}. \bibinfo{publisher}{{AAAI} Press}, \bibinfo{pages}{4503--4511}.
\newblock
\urldef\tempurl%
\url{https://ojs.aaai.org/index.php/AAAI/article/view/16578}
\showURL{%
\tempurl}


\bibitem[Xie et~al\mbox{.}(2022)]%
        {LaGraph}
\bibfield{author}{\bibinfo{person}{Yaochen Xie}, \bibinfo{person}{Zhao Xu}, {and} \bibinfo{person}{Shuiwang Ji}.} \bibinfo{year}{2022}\natexlab{}.
\newblock \showarticletitle{Self-Supervised Representation Learning via Latent Graph Prediction}. In \bibinfo{booktitle}{\emph{International Conference on Machine Learning, {ICML} 2022, 17-23 July 2022, Baltimore, Maryland, {USA}}} \emph{(\bibinfo{series}{Proceedings of Machine Learning Research}, Vol.~\bibinfo{volume}{162})}, \bibfield{editor}{\bibinfo{person}{Kamalika Chaudhuri}, \bibinfo{person}{Stefanie Jegelka}, \bibinfo{person}{Le~Song}, \bibinfo{person}{Csaba Szepesv{\'{a}}ri}, \bibinfo{person}{Gang Niu}, {and} \bibinfo{person}{Sivan Sabato}} (Eds.). \bibinfo{publisher}{{PMLR}}, \bibinfo{pages}{24460--24477}.
\newblock
\urldef\tempurl%
\url{https://proceedings.mlr.press/v162/xie22e.html}
\showURL{%
\tempurl}


\bibitem[Zbontar et~al\mbox{.}(2021)]%
        {BarlowTinws}
\bibfield{author}{\bibinfo{person}{Jure Zbontar}, \bibinfo{person}{Li Jing}, \bibinfo{person}{Ishan Misra}, \bibinfo{person}{Yann LeCun}, {and} \bibinfo{person}{St{\'e}phane Deny}.} \bibinfo{year}{2021}\natexlab{}.
\newblock \showarticletitle{Barlow twins: Self-supervised learning via redundancy reduction}. In \bibinfo{booktitle}{\emph{International Conference on Machine Learning}}. PMLR, \bibinfo{pages}{12310--12320}.
\newblock


\bibitem[Zhang et~al\mbox{.}(2021b)]%
        {CCA-SSG}
\bibfield{author}{\bibinfo{person}{Hengrui Zhang}, \bibinfo{person}{Qitian Wu}, \bibinfo{person}{Junchi Yan}, \bibinfo{person}{David Wipf}, {and} \bibinfo{person}{Philip~S Yu}.} \bibinfo{year}{2021}\natexlab{b}.
\newblock \showarticletitle{From canonical correlation analysis to self-supervised graph neural networks}.
\newblock \bibinfo{journal}{\emph{Advances in Neural Information Processing Systems}}  \bibinfo{volume}{34} (\bibinfo{year}{2021}), \bibinfo{pages}{76--89}.
\newblock


\bibitem[Zhang et~al\mbox{.}(2022a)]%
        {DBLP:conf/iclr/ZhangLSS22}
\bibfield{author}{\bibinfo{person}{Shichang Zhang}, \bibinfo{person}{Yozen Liu}, \bibinfo{person}{Yizhou Sun}, {and} \bibinfo{person}{Neil Shah}.} \bibinfo{year}{2022}\natexlab{a}.
\newblock \showarticletitle{Graph-less Neural Networks: Teaching Old MLPs New Tricks Via Distillation}. In \bibinfo{booktitle}{\emph{{ICLR}}}. \bibinfo{publisher}{OpenReview.net}.
\newblock


\bibitem[Zhang et~al\mbox{.}(2022b)]%
        {COSTA}
\bibfield{author}{\bibinfo{person}{Yifei Zhang}, \bibinfo{person}{Hao Zhu}, \bibinfo{person}{Zixing Song}, \bibinfo{person}{Piotr Koniusz}, {and} \bibinfo{person}{Irwin King}.} \bibinfo{year}{2022}\natexlab{b}.
\newblock \showarticletitle{COSTA: Covariance-Preserving Feature Augmentation for Graph Contrastive Learning}. In \bibinfo{booktitle}{\emph{Proceedings of the 28th ACM SIGKDD Conference on Knowledge Discovery and Data Mining}}. \bibinfo{pages}{2524--2534}.
\newblock


\bibitem[Zhang et~al\mbox{.}(2021a)]%
        {Motif}
\bibfield{author}{\bibinfo{person}{Zaixi Zhang}, \bibinfo{person}{Qi Liu}, \bibinfo{person}{Hao Wang}, \bibinfo{person}{Chengqiang Lu}, {and} \bibinfo{person}{Chee{-}Kong Lee}.} \bibinfo{year}{2021}\natexlab{a}.
\newblock \showarticletitle{Motif-based Graph Self-Supervised Learning for Molecular Property Prediction}. In \bibinfo{booktitle}{\emph{Advances in Neural Information Processing Systems 34: Annual Conference on Neural Information Processing Systems 2021, NeurIPS 2021, December 6-14, 2021, virtual}}, \bibfield{editor}{\bibinfo{person}{Marc'Aurelio Ranzato}, \bibinfo{person}{Alina Beygelzimer}, \bibinfo{person}{Yann~N. Dauphin}, \bibinfo{person}{Percy Liang}, {and} \bibinfo{person}{Jennifer~Wortman Vaughan}} (Eds.). \bibinfo{pages}{15870--15882}.
\newblock
\urldef\tempurl%
\url{https://proceedings.neurips.cc/paper/2021/hash/85267d349a5e647ff0a9edcb5ffd1e02-Abstract.html}
\showURL{%
\tempurl}


\bibitem[Zhu et~al\mbox{.}(2020)]%
        {GRACE}
\bibfield{author}{\bibinfo{person}{Yanqiao Zhu}, \bibinfo{person}{Yichen Xu}, \bibinfo{person}{Feng Yu}, \bibinfo{person}{Qiang Liu}, \bibinfo{person}{Shu Wu}, {and} \bibinfo{person}{Liang Wang}.} \bibinfo{year}{2020}\natexlab{}.
\newblock \showarticletitle{Deep graph contrastive representation learning}.
\newblock \bibinfo{journal}{\emph{arXiv preprint arXiv:2006.04131}} (\bibinfo{year}{2020}).
\newblock


\bibitem[Zhu et~al\mbox{.}(2021)]%
        {GCA}
\bibfield{author}{\bibinfo{person}{Yanqiao Zhu}, \bibinfo{person}{Yichen Xu}, \bibinfo{person}{Feng Yu}, \bibinfo{person}{Qiang Liu}, \bibinfo{person}{Shu Wu}, {and} \bibinfo{person}{Liang Wang}.} \bibinfo{year}{2021}\natexlab{}.
\newblock \showarticletitle{Graph contrastive learning with adaptive augmentation}. In \bibinfo{booktitle}{\emph{Proceedings of the Web Conference 2021}}. \bibinfo{pages}{2069--2080}.
\newblock


\end{thebibliography}

\appendix

\section{Proofs}

\begin{proof}[Proof of corollary~\ref{coro:z1_h1}]
    \label{proof:coro_z1_h1}
    Assuming the loss function Eq.~\eqref{eq:bgrl_loss} reaches the global optimum, we have
    \begin{equation}
        \frac{\widetilde{\mathbf{Z}}_{(1, i)} \widetilde{\mathbf{H}}_{(2, i)}^{\top}}{||\widetilde{\mathbf{Z}}_{(1, i)}||_2 ||\widetilde{\mathbf{H}}_{(2, i)}||_2}=1,
    \end{equation}
    which indicates $\widetilde{\mathbf{Z}}_{(1, i)}$ and $\widetilde{\mathbf{H}}_{(2, i)}$ share the same geometric direction with distinct lengths.
    For convenience, let the length of $\widetilde{\mathbf{Z}}_{(1, i)}$ is $n_{i}$ times of $\widetilde{\mathbf{H}}_{(2, i)}$'s, that is,
    $\widetilde{\mathbf{Z}}_{(1, i)}=n_{i}\space\widetilde{\mathbf{H}}_{(2, i)}$ where $n_{i} > 0$. Taking Eq.~\eqref{eq:same_repre} into consideration, we arrive at
    \begin{equation}\label{eq:z1_h1}
        \widetilde{\mathbf{Z}}_{(1, i)}=\mathbf{W}_p \widetilde{\mathbf{H}}_{(1, i)}=\lambda_{i}\widetilde{\mathbf{H}}_{(1, i)}, where \text{\space} \lambda_{i}=\frac{n_{i}}{m_{i}}>0.
    \end{equation}
    Thus, we arrive at the end of the proof.
\end{proof}

\begin{proof}[Proof of theorem~\ref{theo:infer_predictor}]
    \label{proof:theo_infer_predictor}
    With the above assumptions, we can regard BGRL as a Teacher-Student model, where predictor $\mathbf{W}_p$ represents the student, the teacher network $\mathbf{W}$ is an identity mapping $\mathbf{I}$, and $\mathbf{H}$ represents the input of $\mathbf{W}_p$ and $\mathbf{W}$.
    Accordingly, the graph encoders are considered as a module for preprocessing the original graph.
    From Eq.~\eqref{eq:z1_h1}, we show that $\lambda_{i}$ is affected by both $n_{i}$ and $m_{i}$ while only $n_{i}$ is related to the predictor. Since we are mainly concerned with the predictor, we set $m_{i}=1$ as our assumption described, i.e., $\widetilde{\mathbf{H}}_{1}=\widetilde{\mathbf{H}}_{2}=\mathbf{H}$. For $m_{i}$ with distinct values, we leave it for future work. Therefore, we can derive the following formula,
    \begin{equation}
        \mathbf{Y}_{T} = \mathbf{W} \mathbf{H}=\widetilde{\mathbf{H}}_{2}, \text{   }
        \mathbf{Y}_{S} = \mathbf{W}_{p} \mathbf{H} = \widetilde{\mathbf{Z}}_{1},
    \end{equation}
    where $\mathbf{Y}_{T}$ and $\mathbf{Y}_{S}$ are the outputs of the teacher and student network respectively. We then rewrite the loss function Eq.~\eqref{eq:bgrl_loss} of BGRL as the equivalent one,
    \begin{equation}\label{eq:dis_bgrl_loss}
        \ell^{'}(\theta, \phi) = 2 - \frac{2}{N} \sum_{i=1}^{N}||\widetilde{\mathbf{Y}}_{S(1, i)} - \widetilde{\mathbf{Y}}_{T(2, i)}||_{2}^{2},
    \end{equation}
    where $\widetilde{\mathbf{Y}}_{S(1,i)}$ and $\widetilde{\mathbf{Y}}_{T(2,i)}$ are $\ell_{2}$-normalized vectors. Then, denote the input-output covariance matrix of $\mathbf{W}$ and associated singular value decomposition as follows,
    \begin{equation}
        \sum=\frac{1}{N-1} \mathbf{H}^{\top} \widetilde{\mathbf{Y}}_{T}=\frac{1}{N-1} \mathbf{H}^{\top} \mathbf{H}=\hat{\mathbf{U}}\hat{\mathbf{S}}\hat{\mathbf{V}}.
    \end{equation}
    
    Similarly, the singular value decomposition of $\mathbf{W}_p$ is,
    \begin{equation}
        \mathbf{W}_{p} = \mathbf{U} \mathbf{S} \mathbf{V}.
    \end{equation}
    
    When student $\mathbf{W}_{p}$ is initialized as $\mathbf{W}_p=\epsilon\hat{\mathbf{U}}\hat{\mathbf{V}}^{\top}$ and optimized by Eq.~\eqref{eq:dis_bgrl_loss}, where all student singular values are $\epsilon$, we could apply the training dynamic of Teacher-Student network introduced from ~\cite{TS_Dyna}.
    That is to say, as the training processes, the singular vectors of $\mathbf{W}_p$ remain unchanged while the singular values evolve as
    \begin{equation}
        s(t,\hat{s})=\frac{\hat{s}e^{2\hat{s}t/{\omega}}}{e^{2\hat{s}t/{\omega}} - 1 + \hat{s}/{\omega}},
    \end{equation}
      where $\hat{s}$ and $s$ are singular values from $\hat{S}$ and $S$ respectively, $\omega$ is the reciprocal value of learning rate.

    Hence, when $t \rightarrow \infty$, $s\rightarrow \hat{s}$, and $\mathbf{W}_p\rightarrow \sum$, which is
    \begin{equation}
        \mathbf{W}_p = \frac{1}{N-1}\mathbf{H}^{\top}\mathbf{H}.
    \end{equation}
    Thus, we arrive at the end of the proof.
\end{proof}

\section{More Details on Experiments}\label{appdix:exp_details}
\subsection{Datasets}

For comprehensive comparisons, we validate the quality of node representations on eight public graph benchmarks, including five medium datasets WikiCS~\cite{mernyei2020wiki}, Amazon-Computers, Amazon-Photos, Coauthor-CS, Coauthor-Physics~\cite{shchur2018pitfalls}, and three large-scale datasets ogbn-Arxiv, ogbn-MAG, ogbn-Products~\cite{hu2020ogb}. Dataset statistics are summarized in Table~\ref{tab:dataset}.

\begin{table}[h]
    \caption{Dataset Statistics}
    \label{tab:dataset}
    \scalebox{0.8}{
        \begin{tabular}{lcccc}
            \toprule
            \textbf{Dataset}          & \# Nodes & \# Edges  & \# Features & \# Classes \\
            \midrule
            \textbf{WikiCS}           & 11,701   & 216,123   & 300         & 10         \\
            \textbf{Amazon-Computers} & 13,752   & 245,861   & 767         & 10         \\
            \textbf{Amazon-Photos}    & 7,650    & 119,081   & 745         & 8          \\
            \textbf{Coauthor-CS}      & 18,333   & 81,894    & 6,805       & 15         \\
            \textbf{Coauthor-Physics} & 34,493   & 247,962   & 8,415       & 5          \\
            \textbf{ogbn-Arxiv}       & 169,343  & 1,166,243 & 128         & 40         \\
            \textbf{ogbn-MAG}       & 1,939,743  & 21,111,007 & 128         & 349         \\
            \textbf{ogbn-Products}       & 2,449,029  & 61,859,140 & 100         & 47         \\
            \bottomrule
        \end{tabular}}
\end{table}

\subsection{Implementation}
Following BGRL, the default graph encoder is specified as a two-layer standard GCN~\cite{GCN} followed by a batch normalization~\cite{ioffe2015batch} while a three-layer GCN followed by a layer normalization~\cite{ba2016layer} on ogbn-Arxiv. 
All model parameters are initialized with Glorot initialization~\cite{Glorot}.
To speed up evaluation, we adopt a simple linear layer on CUDA rather than a logistic regression with grid search on CPU used in BGRL. We optimize the graph encoder and linear classifier with AdamW~\cite{loshchilov2017decoupled} and Adam~\cite{Adam} respectively.
For ogbn-Products dataset, applying full-graph training is unrealistic and we perform subgraph-sampling training~\cite{GraphSAGE}. Specifically, we randomly sample 8192 nodes and their neighbors within 2 hops at each training iteration, where 15 neighbors are selected at each hop. 
Since full-graph training is infeasible on the ogbn-Products dataset, we are not able to use the model optimized from the previous iteration to produce the representations of all nodes at once. Therefore, we adopt a node representation cache unit, which stores the representations of all nodes. For each iteration, after sampling a subgraph, we only update the values of the nodes of the subgraph in the cache unit, thus achieving a balance between efficiency and performance.
All experiments are implemented with PyTorch and conducted on a single NVIDIA RTX 3090Ti GPU with 24GB memory. 

\subsection{Baselines}
The comparative methods mainly belong to the following categories: (1) classical semi-supervised GNN algorithms including GCN~\cite{GCN} and GAT~\cite{GAT}, (2) five widely compared GCL methods requiring negative pairs, including DGI~\cite{DGI}, MVGRL~\cite{MVGRL}, GRACE~\cite{GRACE}, GCA~\cite{GCA} and COSTA~\cite{COSTA} and (3) three negative-sample-free GCL methods including BGRL~\cite{BGRL}, AFGRL~\cite{AFGRL}, CCA-SSG~\cite{CCA-SSG}. 
For a more challenging comparison, we also involve a recent generative GSSL advance GraphMAE~\cite{GraphMAE} as a competitor.
For all baselines, we report their official results if available, otherwise, we report the results obtained from their official codes when consistent with our evaluation protocol.

\subsection{Evaluation}
For a fair comparison, we closely follow the linear-evaluation protocol as BGRL. Specifically, we first train the graph encoder in an unsupervised manner. Then, the produced node representations are trained with a $\ell_2$-regularized linear classifier without flowing any gradients back to the graph encoder. In addition to the public divisions for WikiCS and ogb benchmarks, we adopt 10\%:10\%:80\% training/validation/testing random divisions for the remaining datasets. We report the averaged performance over twenty random dataset divisions and model initializations for all datasets apart from ten model initializations for ogbn-Arxiv, ogbn-MAG and ogbn-Products.

\section{Algorithm}\label{appdix:algorithm}
To help better understand the proposed framework, we provide the detailed algorithm for training SGCL in Algorithm~\ref{algo:lightbgrl}.
\begin{algorithm}[h]
    \caption{SGCL training process}
    \label{algo:lightbgrl}
    \begin{algorithmic}[0]
        \Require Graph $\mathcal{G}=(\mathcal{V}, \mathcal{E})$, adjacency matrix $\mathbf{A}$, feature matrix $\mathbf{X}$, graph encoder $f_\theta(\cdot)$, augmentation function $\mathcal{T}$, maximum number of iterations $T$;
        \Ensure The learned encoder $f_\theta(\cdot)$;
    \end{algorithmic}
    \begin{algorithmic}[1] 
        \State Graph augmentation: $\mathbf{A}_{0}, \mathbf{X}_{0} \sim \mathcal{T}(\mathcal{G})$;
        \State Target representation generation: $\mathbf{H}_{0}^{'}=f_{\theta’_{0}}(\mathbf{A}_{0}, \mathbf{X}_{0})$;
        \For{\emph{iteration t $\leftarrow$ $1,\dots, T$}};
        \State Graph augmentation: $\mathbf{A}_{t}, \mathbf{X}_{t} \sim \mathcal{T}(\mathcal{G})$;
        \State Online representation generation: $\mathbf{H}_{t}=f_{\theta_{t}}(\mathbf{A}_{t}, \mathbf{X}_{t})$;
        \State Inferential predictor: calculate $\mathbf{Z}_{t}$ according to Eq \eqref{eq: prediction};
        \State Calculate cosine similarity loss $\mathcal{L}_{\theta_{t}}$ according to Eq \eqref{eq:loss};
        \State Update $\theta_{t}$ to $\theta_{t}^{'}$ by the optimizer;
        \State Target representation generation: $\mathbf{H}'_{t}=f_{\theta_{t}^{'}}(\mathbf{A}_{t}, \mathbf{X}_{t})$
        \EndFor \\
        \Return $f_\theta(\cdot)$;
    \end{algorithmic}
\end{algorithm}

\section{Additional Discussions}

\subsection{Discussions on EMA mechanism}\label{appdix:ema_explain}
It is commonly believed that the EMA update mechanism is indispensable to prevent model collapse.
However, in this paper, we demonstrate that the model still maintains good performance even in the absence of EMA, i.e., $\tau=0$. 
The role of EMA is to offer the possibility of superior performance, as claimed in the BYOL\cite{BYOL}. 
Namely, different values of $\tau$ enable the model parameters to be integrated from the previous training steps, thus enhancing the performance. 
Nevertheless, based on Figure~\ref{img:module_cmp_wikics}, we reveal that EMA makes a negligible contribution to BGRL.
Actually, the two crucial modules that prevent BGRL from collapsing are the predictor and stop-gradients, where the latter has been highlighted as an important training technique in SimSiam\cite{Siamese}. 
In this paper, we concentrate more on understanding the role of different components in BGRL architecture and the reason why BGRL can produce discriminative representations without negative samples, thus obtaining a more concise and effective framework.


\end{document}